%% file: ms.tex
\documentclass[fleqn,12pt]{article} 
\pdfoutput=1
\usepackage{a4wide}
\usepackage{mathtools}
\usepackage{geometry}
\usepackage{graphicx}
\usepackage{amssymb}
\usepackage{amsmath}
\usepackage{amsthm}
\usepackage{dsfont}
\usepackage{color}
\usepackage{bm}
\usepackage{extarrows}
\usepackage{chngcntr}
\usepackage{enumitem}
\usepackage{multirow}
\usepackage{makecell}
\usepackage{breakcites}
\usepackage[utf8]{inputenc}
\usepackage[english]{babel}
\usepackage[hidelinks]{hyperref}
\usepackage{cleveref}
\usepackage{calc}
\usepackage{textgreek}

\input{Makros}

\title{\textbf{On the Connection between Lp- and Risk Consistency and its Implications on Regularized Kernel Methods}}
\date{March 27, 2023}
\author{\textbf{Hannes K\"ohler}\thanks{Email: \href{mailto:hannes.koehler@uni-bayreuth.de}{\texttt{hannes.koehler@uni-bayreuth.de}}}\\
		Department of Mathematics, University of Bayreuth, Germany
	}

\begin{document}
\maketitle

\begin{abstract}
	As a predictor's quality is often assessed by means of its risk, it is natural to regard risk consistency as a desirable property of learning methods, and many such methods have indeed been shown to be risk consistent. The first aim of this paper is to establish the close connection between risk consistency and \Lp-consistency for a considerably wider class of loss functions than has been done before. The attempt to transfer this connection to shifted loss functions surprisingly reveals that this shift does not reduce the assumptions needed on the underlying probability measure to the same extent as it does for many other results. The results are applied to regularized kernel methods such as support vector machines.
	
	\vspace*{1ex}\noindent\textbf{Keywords:} machine learning, consistency, regression, kernel methods, support vector machines
\end{abstract}

\section{Introduction}\label{Sec:Consistency_Intro}

The goal of non-parametric statistical machine learning is to predict an output random variable $Y$ based on an input random variable $X$ with (almost) no prior knowledge about the distribution \P of $(X,Y)$ on some space \XY, all information about \P typically stemming from a data set $\Dn:= ((x_1,y_1),\dots,(x_n,y_n))\in (\XY)^n$ consisting of independent and identically distributed (\iid) observations sampled from \P. More specifically, one aims at finding a measurable function $f\colon\X\to\Y$ which captures certain characteristics of the conditional distribution \Pbed{X}, like its conditional mean function or conditional quantile function. 

Such learning tasks can often be formalized by aiming at finding a measurable function that minimizes the \textit{\loss-risk} (or just \textit{risk}) 
\begin{align*}
	\risk(f) := \ew{\loss(X,Y,f(X))}\,
\end{align*}
for a suitable \textit{loss function}, which is a measurable function $\loss\colon\XYR\to[0,\infty)$. Here, $\loss(x,y,f(x))$ quantifies the cost of the prediction $f(x)$ if the observed true output belonging to $x$ is $y$. Hence, the choice of \loss controls how different deviations between $y$ and $f(x)$ are penalized and specifies the exact goal of the prediction, and the risk assesses the quality of the whole predictor $f$ with respect to the whole distribution \P. For example, the two aforementioned goals of finding the conditional means (least squares regression) and conditional quantiles (quantile regression) can be approached by using the least squares loss and the pinball loss respectively, as it is known that the according risks are minimized by the respective target functions one aims to estimate.

To this end, we define the \textit{Bayes risk} \riskbayes as usual as the smallest possible risk, that is,
\begin{align*}
	\riskbayes := \inf\{\risk(f) \,|\, f\colon\X\to\R \text{ measurable}\}\,,
\end{align*}
and call a measurable function \fbayes achieving $\risk(\fbayes)=\riskbayes$ a \textit{Bayes function}. Further assume that a learning method yields the predictor $\fn$ based on the data set $\Dn$, $n\in\N$.

Because of the risk assessing a predictor's quality, a desirable property for the learning method is \textit{risk consistency}, \ie that
\begin{align*}
	\risk(\fn) \to \riskbayes\,,\qquad n\to\infty\,,
\end{align*}
in probability. As this is a very natural type of consistency to consider, results on risk consistency exist for many such learning methods, see for example \citet{steinwart2005} (regularized kernel methods for classification), \citet{zhang2005} (boosting), \citet{christmann2007} (regularized kernel methods for regression; see also \Cref{Sec:Consistency_SVMs}), \citet{biau2008} (averaging classifiers such as random forests), \citet{lin2022} (deep convolutional neural networks). 

We are however also interested in taking a look at a different type of consistency, namely \textit{\Lp-consistency}, \ie that
\begin{align*}
	\normLppx{\fn-\fbayes} \to 0\,,\qquad n\to\infty\,,
\end{align*}
in probability for some $p\in[1,\infty)$, as this compares the functions themselves instead of their risks.

We show in \Cref{Sec:Consistency_Consistency} that \Lp- and risk consistency are actually equivalent under rather mild assumptions. Here, the more surprising part certainly is risk consistency implying \Lp-consistency as the latter tackles the generally more demanding task of system identification instead of only system imitation, as it is described by \citet[Section~2.1.1]{cherkassky2007}, see also \citet[Section~1.4]{gyorfi2002} for the classification case. Whereas this implication had already been established for certain special loss functions (which we briefly recap in \Cref{SubSec:Consistency_Consistency_Reg}), we considerably generalize it to a large class of loss functions including those as special cases. Additionally, we examine whether it is possible to transfer these results to risks that are based on shifted loss functions---which are useful for working with heavy-tailed distributions---and stumble upon some difficulties when trying to do this in all generality, which is somewhat surprising considering that many other results can be transferred to shifted loss functions quite seamlessly. Before successfully transferring our results by imposing some assumptions on the underlying distribution, we therefore also derive some interesting negative results. Lastly, in \Cref{Sec:Consistency_SVMs}, the consistency results are applied to regularized kernel methods, in which the predictors are defined as minimizers of regularized risks. Because of this definition, it is natural to examine their risk consistency and this has already been well investigated in the past, but there did not exist any general results on their \Lp-consistency so far.

We wish to emphasize that our goal is not to derive learning rates for any learning method (like for example for the regularized kernel methods from \Cref{Sec:Consistency_SVMs}). Instead, we aim at deriving results on consistency under minimal assumptions on the underlying probability distribution---much weaker assumptions than those needed for deriving learning rates---and such that the results are applicable to general learning methods in a general setting.

\section{Prerequisites}\label{Sec:Consistency_Pre}

Before presenting our results, we first need to state some additional prerequisites: As mentioned in the introduction, we aim at estimating certain properties of the unknown conditional distribution \Pbed{X} such as the conditional mean or conditional quantiles. This conditional distribution \Pbed{X} uniquely exists, and \P can therefore be split into a marginal distribution \Px on \X and this conditional distribution, whenever \Y is a Polish space \citep[\vgl][Theorems 10.2.1 and 10.2.2]{dudley2004}, for example if $\Y\subseteq\R$ closed \citep[\vgl][\seite 157]{bauer2001}. Hence, by choosing \Y in such a way, we are guaranteed to always be able to perform this factorization of \P, which leads us to one part of the following standard and rather general assumption which we assume to hold true throughout this paper.

\begin{myann}\label{Ann:Consistency_Pre_AllgAnn}
	Let \X be a complete separable metric space and let $\Y\subseteq\R$ be closed. Let \X and \Y be equipped with their respective Borel $\sigma$-algebras \BX and \BY. Let $\P\in\MXY$, where \MXY denotes the set of all Borel probability measures on the measurable space $(\XY,\BXY)$.
\end{myann}

We are mainly interested in continuous and in convex loss functions, by which we mean continuity respectively convexity of \loss in its third argument. Furthermore, the loss functions will be assumed to additionally be {distance-based}. Distance-based losses are a special type of loss functions which are typically used in regression tasks,
and which are defined in the following way:

\begin{mydef}\label{Def:Consistency_Pre_DistBasedLoss}
	A loss function $\loss\colon\XYR\to[0,\infty)$ is called \textit{distance-based} if there exists a representing function $\psi\colon \R\to[0,\infty)$ satisfying $\psi(0)=0$ and $\loss(x,y,t)=\psi(y-t)$ for all $(x,y,t)\in\XYR$. If $\psi(r)=\psi(-r)$ for all $r\in\R$, then \loss is called \textit{symmetric}.
	\\
	Let $p\in(0,\infty)$. A distance-based loss $\loss\colon\XYR\to[0,\infty)$ with representing function $\psi$ is of
	\begin{enumerate}[label=(\roman*)]
		\item \textit{upper growth type} $p$ if there is a constant $c>0$ such that $\psi(r) \le c\, (|r|^p+1)$ for all $r\in\R$.
		\item \textit{lower growth type} $p$ if there is a constant $c>0$ such that $\psi(r) \ge c\,|r|^p - 1$ for all $r\in\R$.
		\item \textit{growth type} $p$ if \loss is of both upper and lower growth type $p$.
	\end{enumerate}
\end{mydef}

Since the first argument does not matter in distance-based loss functions, we often ignore it and write $\loss\colon\YR\to[0,\infty)$ and $\loss(y,t)$ instead.

Distance-based losses are typically used in regression tasks, but some of them, like the least squares loss, are also popular choices for classification tasks, see for example \citet[Section~1.4]{gyorfi2002}. As an example of a distance-based loss, the mentioned least squares loss is of growth type 2 whereas many other common loss functions for regression tasks, like the {pinball loss}, {Huber loss} or {\eps-insensitive loss}, are of growth type 1. We will later see that this sometimes leads to slightly more restrictive conditions regarding \P when using the least squares loss.

More specifically, it is for some results required that the \textit{averaged $p$-th moment} of \P, with $p$ being the loss function's growth type, is finite. This averaged $p$-th moment is defined as
\begin{align*}
	|\P|_p := \left(\int_{\XY} |y|^p \diff\P(x,y)\right)^{1/p} &= \left(\int_{\X}\int_{Y} |y|^p \diff\Pbed[y]{x}\diff\Px(x)\right)^{1/p}\,.
\end{align*}

\section{Connection between Lp- and Risk Consistency}\label{Sec:Consistency_Consistency}

In \Cref{SubSec:Consistency_Consistency_Reg}, we show that \Lp- and risk consistency are equivalent under certain conditions. \Cref{SubSec:Consistency_Consistency_Shift} contains the rather surprising result that some of these results can not be transferred to risks that are based on shifted loss functions in the generality we would have hoped for, but we also introduce some additional conditions under which it is possible to transfer the results after all.

\begin{mybem}\label{Bem:Consistency_Consistency_BayesUnique}
	We will often write ``\textit{the} Bayes function", implying there exists exactly one such measurable function minimizing \risk. This does not always hold true and is not necessary for risk consistency (neither existence nor uniqueness). We however assume that the Bayes function indeed exists and is \Px-almost-surely (\fs) unique whenever we investigate the difference between some predictor and the Bayes function directly (\eg in the results on \Lp-consistency) instead of the difference between the according risks.
\end{mybem}

\subsection{Connection between Lp- and risk consistency for regular loss functions}\label{SubSec:Consistency_Consistency_Reg}

So far, there are no general results on \Lp-consistency following from risk consistency, but only results regarding special loss functions: For the least squares loss, it has been known for many years that a function's excess risk, \ie the difference between its risk and the Bayes risk, corresponds to the squared $\L{2}(\Px)$-norm of its deviation from the Bayes function, and risk consistency therefore implies $\L{2}$-consistency, \vgl \citet[Proposition 1]{cucker2001} or \citet[\seiten 26--28]{cherkassky2007}. Recently, this $\L{2}$-difference between a function and the Bayes function has also been bounded by the excess risk---by means of so-called comparison or self-calibration inequalities---in case of the asymmetric least squares loss by \citet{farooq2019} and in case of more general strongly convex loss functions under additional assumptions by \citet{sheng2020}. Additionally, \citet{hable2014} showed that \La-consistency follows from risk consistency in case of the pinball loss, and \citet{steinwart2011,xiang2012} derived self-calibration inequalities for this loss under additional assumptions. \citet{tong2019} did so for the \eps-insensitive loss. 

The following lemma generalizes the aforementioned special cases to general convex, distance-based loss functions:

\begin{mythm}\label{Thm:Consistency_Consistency_RiskLp}
	Let $\loss\colon\YR\to[0,\infty)$ be a convex, distance-based loss function of lower growth type $p\in[1,\infty)$. Assume that $\fbayes$ is \Px-\fs unique, $\fbayes\in \Lppx$ and $\riskbayes<\infty$. Then, for every sequence $(\fn)_{n\in\N}\subseteq \Lppx$, we have
	\begin{equation*}
		\limn \risk(\fn) = \riskbayes\qquad \Rightarrow\qquad  \limn ||\fn - \fbayes||_{\Lppx} = 0\,.
	\end{equation*}
\end{mythm}

\begin{mybem}\label{Bem:Consistency_Consistency_RiskLpsAlternativbedingungen}
	If \loss is of growth type $p$ instead of only being of \textit{lower} growth type $p$, the conditions $\fbayes\in \Lppx$ and $\riskbayes<\infty$ in \Cref{Thm:Consistency_Consistency_RiskLp} can also be replaced by the perhaps more intuitive and in this case equivalent \textit{moment condition} $|\P|_p<\infty$. This equivalence can easily be obtained from parts (i) and (iii) of \citet[Lemma~2.38]{steinwart2008} by noting that $\riskbayes\le\risk(0)$, with $0$ denoting the zero function, always holds true by definition of the Bayes risk.
\end{mybem}

Notably, \Cref{Thm:Consistency_Consistency_RiskLp} strengthens \citet[Corollary~3.62]{steinwart2008}, 
which stated that risk consistency implies weak consistency.

As mentioned in the introduction, the opposite direction---risk consistency following from \Lp-consistency---is generally the easier one. We formally state this implication in the subsequent \Cref{Thm:Consistency_Consistency_LpRisk}. Hence, this theorem can be seen as the counterpart of \Cref{Thm:Consistency_Consistency_RiskLp}, even though the conditions of the two theorems differ in some details. Notably, the function $f^*$, which the sequence is converging to, does not necessarily need to be the Bayes function \fbayes here:

\begin{mythm}\label{Thm:Consistency_Consistency_LpRisk}
	Let $\loss\colon\YR \to [0,\infty)$ be a continuous, distance-based loss function of upper growth type $p\in[1,\infty)$. Assume that $|\P|_p<\infty$. Then, for every sequence $(f_n)_{n\in\N} \subseteq \Lppx$ and every function $f^*\in\Lppx$, we have
	\begin{align*}
		\limn \normLppx{f_n-f^*} = 0 \qquad \Rightarrow \qquad \limn \risk(f_n) = \risk(f^*)\,.
	\end{align*}
\end{mythm}

\subsection{Connection between Lp- and risk consistency for shifted loss functions}\label{SubSec:Consistency_Consistency_Shift}

When looking at \Cref{Thm:Consistency_Consistency_RiskLp}, it is obvious that the assumptions $\fbayes\in\Lppx$ and $\riskbayes<\infty$ are indeed necessary for the theorem's conclusion and that one cannot hope to derive \Lp- from risk consistency without them. Because these assumptions are equivalent to $|\P|_p<\infty$ if \loss is of growth type $p$ (\vgl \Cref{Bem:Consistency_Consistency_RiskLpsAlternativbedingungen}), this however excludes heavy-tailed distributions such as the Cauchy distribution---even for $p=1$. Analogously, \Cref{Thm:Consistency_Consistency_LpRisk} also requires $|\P|_p<\infty$ and can therefore not be applied to such heavy-tailed distributions.

To circumvent this problem, we now try to transfer the results from \Cref{SubSec:Consistency_Consistency_Reg} to \textit{shifted loss functions}, which have been applied in robust statistics for a long time, see for example \citet{huber1967} or \citet[Chapter~3]{huber2009}, and which can be defined in a very easy way: Given a loss function $\loss\colon \XYR\to[0,\infty)$, the associated shifted loss function is
\begin{align*}
	\lossshift\colon&\ \XYR \to \R\,,\,\\
	&(x,y,t)\mapsto \loss(x,y,t)-\loss(x,y,0)\,,
\end{align*}
which can be used to estimate the same quantities as the original loss function since the shift is fixed independently of $t$. Risks can be defined in the same way as for regular loss functions.

\begin{mybem}\label{Bem:Consistency_Consistency_Lipschitz}
	By \citet[Lemma~2.34]{steinwart2008}, a convex and distance-based loss function of upper growth type 1 is always Lipschitz continuous. We call a loss function \loss \textit{Lipschitz continuous} if it is Lipschitz continuous with respect to its last argument, that is, if
	\begin{align*}
		|\loss(x,y,t)-\loss(x,y,t')| \le |\loss|_1 \cdot |t-t'| \qquad \forall\, (x,y)\in\XY\,,\, t,t'\in\R\,,
	\end{align*}
	for some constant $|\loss|_1\ge 0$ which is called the \textit{Lipschitz constant} of \loss.
\end{mybem}

With \Cref{Bem:Consistency_Consistency_Lipschitz} in mind, the risk with respect to the shifted version of a convex and distance-based loss function of upper growth type 1 can be bounded by
\begin{align}\label{eq:Consistency_Consistency_ShiftRiskFinite}
	|\riskshift(f)| \le \int_{\XY} |\loss(y,f(x))-\loss(y,0)|\diff\P(x,y) \le |\loss|_1 \int_\X |f(x)| \diff\Px(x)\,.
\end{align}
Hence, even if $|\P|_1=\infty$, this risk is finite for all $f\in\Lapx$. Using the shifted loss therefore seems like a promising approach for extending the applicability of the results from \Cref{SubSec:Consistency_Consistency_Reg} to heavy-tailed distributions and getting rid of the moment condition $|\P|_1<\infty$ in the case of having a convex loss function of growth type 1. Indeed, \citet{christmann2009} showed that the moment condition can in this case be eliminated from many results regarding regular loss functions by transferring them to shifted loss functions.

When looking at the proof of \Cref{Thm:Consistency_Consistency_RiskLp}, it is however easy to see that \eqref{eq:ProofThm_Consistency_Consistency_RiskLp_LGleichgrInt} does not hold true for shifted loss functions and the proof can thus not be transferred to the situation of this section. The following negative result shows that this is indeed not a failing of the specific proof we used, but that \La-consistency does, somewhat surprisingly, actually not follow from \lossshift-risk consistency in the generality one would have hoped for:

\begin{myprop}\label{Prop:Consistency_Consistency_ShiftGegenbeispielRiskLp}
	Let $\Y=\R$. Let $\loss\colon\YR\to[0,\infty)$ be a convex, distance-based and symmetric loss function of growth type 1, and let \lossshift be its shifted version. Then, even if \fshiftbayes is \Px-\fs unique with $\fshiftbayes\in\Lapx$, a sequence $(\fn)_{n\in\N}\subseteq\Lapx$ of functions satisfying 
	\begin{align*}
		\limn \riskshift(\fn) = \riskshiftbayes
	\end{align*}
	does in general \textbf{not} imply 
	\begin{align*}
		\limn \normLapx{\fn - \fshiftbayes} = 0
	\end{align*}
	without any additional assumptions besides \Cref{Ann:Consistency_Pre_AllgAnn} being imposed.
\end{myprop}

Note that in the situation of \Cref{Prop:Consistency_Consistency_ShiftGegenbeispielRiskLp}, risk consistency does also not imply \Lp-consistency for any $p>1$ since \Lp-consistency for $p>1$ would imply \La-consistency.

We now take a special look at the \textit{$\tau$-pinball loss} (or just \textit{pinball loss})
\begin{align}\label{eq:Consistency_Consistency_Pinball}
	\losspin\colon&\ \YR\to [0,\infty)\,, \notag\\
	&\ (y,t) \mapsto \begin{cases}
		(1-\tau)\cdot(t-y) & \wenn y<t\,,\\
		\tau\cdot(y-t) & \wenn y\ge t\,,
	\end{cases}
\end{align}
$\tau\in(0,1)$, which is convex and distance-based with growth type 1, but not symmetric for $\tau\ne 0.5$. As mentioned in the introduction, the pinball loss can be used for quantile regression, \ie for estimating the conditional quantiles
\begin{align*}
	F_{\tau,\P}^*\colon&\ \X\to 2^\R\,,\notag\\ 
	&\ x\mapsto \{t^*\,\left|\, \P((-\infty,t^*]|x)\ge \tau\text{ and }\P([t^*,\infty)|x)\ge 1-\tau \right.\}\,,
\end{align*}
see also \citet{koenker1978,koenker2001,takeuchi2006,steinwart2011}.

If one assumes these conditional quantiles $F_{\tau,\P}^*(x)$ to \Px-\fs be singletons, it is possible to denote them by the \Px-\fs unique quantile function $\ftaubayes\colon \X\to \R$ defined by $\{\ftaubayes(x)\}= F_{\tau,\P}^*(x)$ for all $x\in\X$. Recall that
this \ftaubayes is the up to \Px-zero sets only measurable function satisfying 
\begin{align}\label{eq:Consistency_Consistency_PinballBayesRisk}
	\risk[\losspin,\P](\ftaubayes)=\riskbayes[\losspin,\P]\,
\end{align}
if $\riskbayes[\losspin,\P]$ is finite, and similarly, that \ftaubayes satisfies
\begin{align}\label{eq:Consistency_Consistency_ShiftPinballBayesRisk}
	\risk[\losspinshift,\P](\ftaubayes)=\riskbayes[\losspinshift,\P]
\end{align}
and is the up to \Px-zero sets only measurable function doing so if $\riskbayes[\losspinshift,\P]$ is finite.
This ties our assumption of the conditional quantiles \Px-\fs being singletons to \Cref{Bem:Consistency_Consistency_BayesUnique} about the required \Px-\fs uniqueness of the Bayes function and yields $\ftaubayeswithshiftloss\equiv\ftaubayes$ \Px-\fs

As non-symmetric loss functions are not covered by \Cref{Prop:Consistency_Consistency_ShiftGegenbeispielRiskLp} and as the pinball loss is the probably most popular among these, we specifically investigate this loss function's behavior and obtain the following analogous result to \Cref{Prop:Consistency_Consistency_ShiftGegenbeispielRiskLp}:

\begin{myprop}\label{Prop:Consistency_Consistency_ShiftGegenbeispielPinRiskLp}
	Let $\Y=\R$. Let $\tau\in(0,1)$ and let \losspinshift be the shifted version of the $\tau$-pinball loss.\footnote{It can easily be seen that this shifted pinball loss function is, for $\tau\in(0,1)$, \begin{align*}
			\losspinshift\colon&\ \YR \to \R\\
			&\ (y,t) \mapsto \losspin(y,t) - \losspin(y,0) =  \begin{cases}
				(1-\tau)\cdot t &\wenn y<\min\{0,t\}\,,\\
				(1-\tau)\cdot t - y &\wenn 0\le y<t\,,\\
				y - \tau\cdot t &\wenn t\le y<0\,,\\
				- \tau\cdot t &\wenn y\ge \max\{0,t\}\,.
			\end{cases}
	\end{align*}} Then, even if \ftaubayes is \Px-\fs unique with $\ftaubayes\in\Lapx$, a sequence $(\fn)_{n\in\N}\subseteq\Lapx$ of functions satisfying 
	\begin{align*}
		\limn \risk[\losspinshift,\P](\fn) = \risk[\losspinshift,\P]^*
	\end{align*}
	does in general \textbf{not} imply 
	\begin{align*}
		\limn \normLapx{\fn - \ftaubayes} = 0
	\end{align*}
	without any additional assumptions besides \Cref{Ann:Consistency_Pre_AllgAnn} being imposed.
\end{myprop}

As the preceding results allow for arbitrary sequences of functions in \Lapx, we might still hope to deduce \La-consistency following from \lossshift-risk consistency by restricting ourselves to smaller function spaces with more structure like Sobolev spaces. However, the subsequent corollary shows that \Cref{Prop:Consistency_Consistency_ShiftGegenbeispielRiskLp} and \Cref{Prop:Consistency_Consistency_ShiftGegenbeispielPinRiskLp} can even be strengthened to sequences of functions from Sobolev spaces. Here, we assume that $\X\subseteq\R^d$ open for some $d\in\N$, and we denote by \sobolevX the Sobolev space consisting of all functions from $\L{q}(\X)$ whose weak derivatives up to order $m$ are also in $\L{q}(\X)$, \vgl \citet[Definition~3.2]{adams2003}. Here, as usual, $\L{q}(\X)$ denotes the $\L{q}$-space with respect to the Lebesgue measure on \X.

\begin{mykor}\label{Cor:Consistency_Consistency_ShiftGegenbeispielSobolevRiskLp}
	Let $d\in\N$, $\X\subseteq\R^d$ open, and $\Y=\R$. Let $\loss\colon\YR\to[0,\infty)$ be a convex, distance-based and symmetric loss function of growth type 1, or the $\tau$-pinball loss for some $\tau\in(0,1)$. Let \lossshift be its shifted version. Let $m\in\N$ and $1\le q\le\infty$. Then, even if \fshiftbayes is \Px-\fs unique with $\fshiftbayes\in \Lapx$, a sequence $(f_n)_{n\in\N}\subseteq\sobolevX\cap\Lapx$ of functions satisfying 
	\begin{align*}
		\limn \riskshift(f_n) = \riskshiftbayes
	\end{align*}
	does in general \textbf{not} imply 
	\begin{align*}
		\limn \normLapx{f_n - \fshiftbayes} = 0
	\end{align*}
	without any additional assumptions besides \Cref{Ann:Consistency_Pre_AllgAnn} being imposed.
\end{mykor}

The preceding results show that it is not possible to get rid of the moment condition from \Cref{Thm:Consistency_Consistency_RiskLp} (\vgl \Cref{Bem:Consistency_Consistency_RiskLpsAlternativbedingungen}) just by transferring it to shifted loss functions. It might, however, still be possible to circumvent this moment condition by instead imposing some different and less restrictive conditions. For the pinball loss from \eqref{eq:Consistency_Consistency_Pinball}, \ie for doing quantile regression, we are indeed able to derive such an alternative and in many cases less restrictive condition regarding \P. To be more specific, the conditional distribution \Pbed{X} is, in some sense, not allowed to be too heteroscedastic and it has to be continuous in the conditional quantiles $\ftaubayes(x)$, $x\in\X$: 

\begin{mythm}\label{Thm:Consistency_Consistency_ShiftPinRiskLa}
	Let $\tau\in(0,1)$ and \losspinshift be the shifted version of the $\tau$-pinball loss. Assume that \ftaubayes is \Px-\fs unique, $\ftaubayes\in \Lapx$, and \P additionally satisfies at least one of the following conditions:
	\begin{enumerate}[label=(\roman*)]
		\item $|\P|_1 < \infty$.
		\item There exist $c_1,c_2>0$ such that
		\begin{align}\label{eq:Thm_Consistency_Consistency_ShiftPinRiskLa_WMasseHerum}
			&\P\Big((\ftaubayes(X)-c_1,\ftaubayes(X))\,\big|\,X\Big) \ge c_2 \,\,\text{ and }\,\,\notag\\
			&\P\Big((\ftaubayes(X),\ftaubayes(X)+c_1)\,\big|\,X\Big) \ge c_2
		\end{align}
		\Px-\fs, as well as 
		\begin{align}\label{eq:Thm_Consistency_Consistency_ShiftPinRiskLa_WMasse0}
			\Pbed[\ftaubayes(X)]{X}=0
		\end{align}
		\Px-\fs
	\end{enumerate}
	Then, for every sequence $(\fn)_{n\in\N}\subseteq \Lapx$, we have
	\begin{align*}
		\limn \risk[\losspinshift,\P](\fn) = \risk[\losspinshift,\P]^*  \qquad \Rightarrow\qquad  \limn ||\fn - \ftaubayes||_{\Lapx} = 0\,.
	\end{align*}
\end{mythm}

Even though it was not possible to get rid of the moment condition (i) without imposing the new condition (ii), this still substantially expands the theorem's applicability since there are many cases in which (ii) (whose first part is visualized in \Cref{Abb:Consistency_Consistency_VisualizationC1C2}) is satisfied even though (i) is not:

\begin{figure}[t]
	\begin{center}
		\vspace*{-0.5cm}\includegraphics[width=0.9\textwidth]{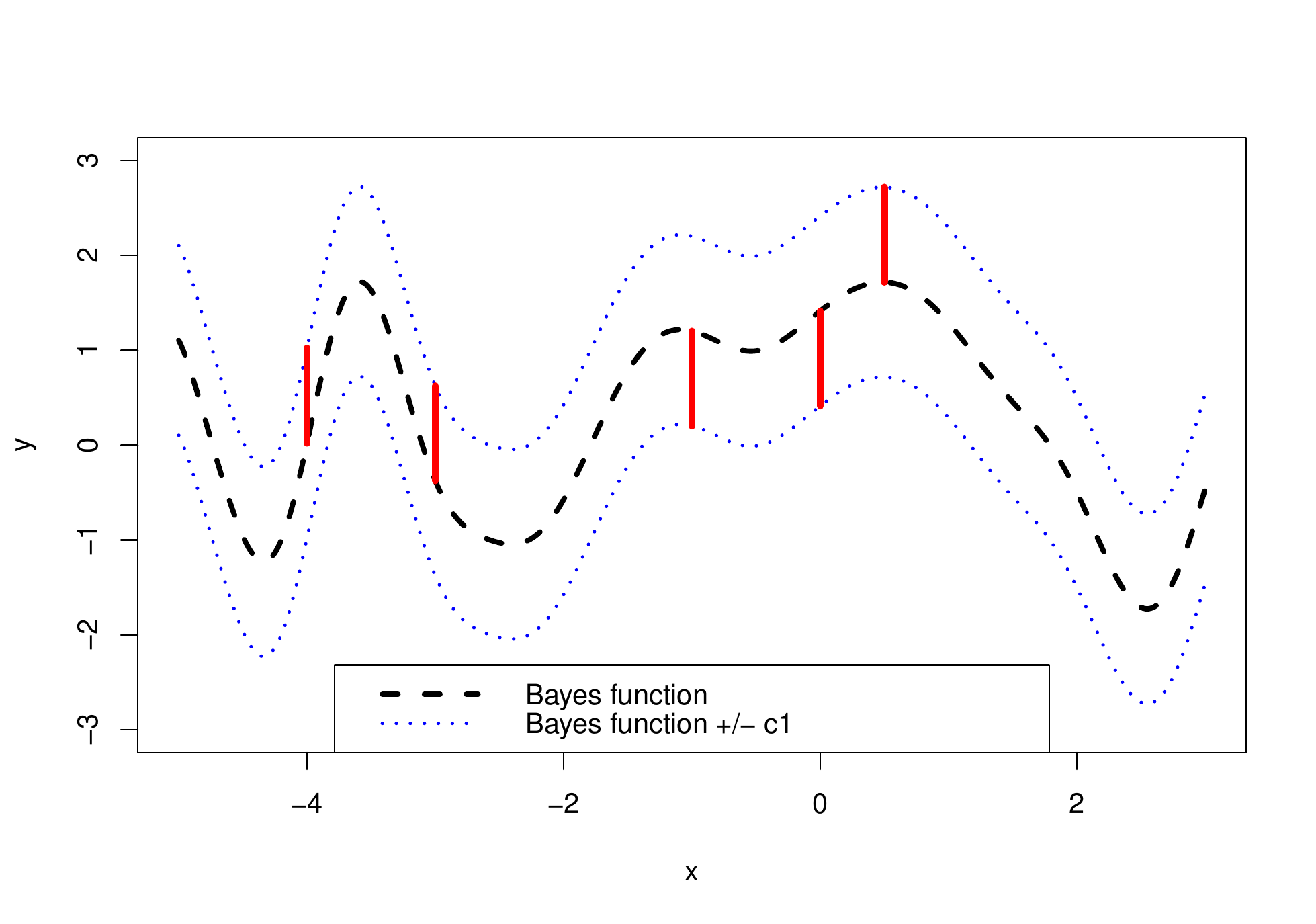}
		\vspace*{-0.5cm}\caption{Visualization of \eqref{eq:Thm_Consistency_Consistency_ShiftPinRiskLa_WMasseHerum}. Each vertical slice between $\ftaubayes-c_1$ and \ftaubayes as well as between \ftaubayes and $\ftaubayes+c_1$ needs to have a conditional probability (given $x$) of at least $c_2$. The solid vertical lines depict some examples of such slices whose conditional probability needs to be at least $c_2$.}
		\label{Abb:Consistency_Consistency_VisualizationC1C2}
	\end{center}
\end{figure}

\begin{mybsp}\label{Bsp:Consistency_Consistency_HomoscedasticRegression}
	Assume that $\tau\in(0,1)$ and that we have an underlying homoscedastic regression model like
	\begin{align*}
		Y = f(X) + \eps\,,
	\end{align*} 
	where $f\colon\X\to\Y$ is an arbitrary measurable function and \eps is a continuous random variable whose distribution does not depend on the value of $X$. Whenever \eps has a unique $\tau$-quantile $q_\tau\in\R$, (ii) from \Cref{Thm:Consistency_Consistency_ShiftPinRiskLa} holds true with $\ftaubayes=f+q_\tau$. For example, \eps can follow a Cauchy distribution with location and scale parameters which are fixed independently of the value of $X$. In this case, the moment condition (i) does not hold true, but \Cref{Thm:Consistency_Consistency_ShiftPinRiskLa} does still yield \La-consistency following from risk consistency.
\end{mybsp}

\begin{mybsp}
	The independence of \eps from $X$ in \Cref{Bsp:Consistency_Consistency_HomoscedasticRegression} is not even strictly necessary. Assume the more general heteroscedastic model
	\begin{align*}
		Y = f(X) + \eps_X\,,
	\end{align*}
	where the distribution of $\eps_X$ is now allowed to depend on the value $x$ of $X$. If, for example, there exist $C>0$ and $c_1>0$ such that $\eps_x$ has a unique $\tau$-quantile $q_{x,\tau}\in\R$ and Lebesgue density greater than $C$ on $(q_{x,\tau}-c_1,q_{x,\tau}+c_1)$ for \Px-almost all $x\in\X$, condition (ii) from \Cref{Thm:Consistency_Consistency_ShiftPinRiskLa} is still satisfied. 
	
	For example, this situation is on hand if $\X=\R^d$ for some $d\in\N$, $\Y=\R$, and $\eps_x$ follows a Cauchy distribution with location parameter $\cos(||x||_2)$ and scale parameter $2+\sin(||x||_2)$ for all $x\in\X$. More generally, the same also holds true for different choices of location and scale parameters, as long as they are bounded from above and from below (in the case of the scale parameter we mean bounded away from zero by bounded from below).
\end{mybsp}

We saw that \La-consistency can not be obtained from risk consistency without imposing some different, albeit in some sense weaker, condition regarding \P in exchange for omitting the moment condition. It is, however, indeed possible to just omit the moment condition in the reverse statement (\Cref{Thm:Consistency_Consistency_LpRisk}) when transferring this to shifted loss functions in the case of having a convex loss function of upper growth type 1, which again hints at this direction being the easier one as it was mentioned in the introduction.

\begin{mythm}\label{Thm:Consistency_Consistency_Shift_L1Risk}
	Let $\loss\colon\Y\to\R$ be a convex, distance-based loss function of upper growth type 1, and let \lossshift be its shifted version. Then, for every sequence $(\fn)_{n\in\N}\subseteq\Lapx$ and every function $\fstar\in\Lapx$, we have
	\begin{align*}
		\limn \normLapx{\fn-\fstar} = 0 \qquad \Rightarrow \qquad \limn \riskshift(\fn) = \riskshift(\fstar)\,.
	\end{align*}
\end{mythm}

\section{Consistency of Regularized Kernel Methods}\label{Sec:Consistency_SVMs}

After having derived general results regarding the connection between \Lp- and risk consistency in \Cref{Sec:Consistency_Consistency}, we would like to apply these results to special predictors now. More specifically, we investigate kernel-based regularized risk minimizers, which we also call \textit{support vector machines (SVMs)}. We are thus using the term SVM in a broad sense, allowing not only for the hinge loss (as the expression SVM is used in some works) but rather for arbitrary loss functions including the distance-based losses used in \Cref{Sec:Consistency_Consistency}.

We first give a formal definition and some further mathematical prerequisites regarding SVMs as well as a short recap of some of their known properties in \Cref{SubSec:Consistency_SVMs_Pre}. In \Cref{SubSec:Consistency_SVMs_Reg}, we then first use our results from \Cref{SubSec:Consistency_Consistency_Reg} to derive a result on their \Lp-consistency, where no general result existed so far, and then derive a new result on their risk consistency, which in some part slightly weakens the conditions from existing results on risk consistency. Finally, we examine SVMs based on shifted loss functions in \Cref{SubSec:Consistency_SVMs_Shift}.

\subsection{Prerequisites regarding regularized kernel methods}\label{SubSec:Consistency_SVMs_Pre}

As the true distribution \P is usually unknown in practice, one has to make do with the information available about \P, \ie the data set $\Dn:=((x_1,y_1),\dots,(x_n,y_n))\in(\XY)^n$ mentioned in the introduction and consisting of \iid\@ observations sampled from \P, instead of minimizing \risk directly. This is approached by using the empirical distribution 
\begin{align*}
	\DVertn := \frac{1}{n}  \sum_{i=1}^{n} \delta_{(x_i,y_i)}\,,
\end{align*}
corresponding to \Dn, with $\delta_{(x_i,y_i)}$ denoting the Dirac measure in $(x_i,y_i)$, and defining the \textit{empirical risk} \riskempn analogously to \risk, which results in
\begin{align*}
	\riskempn(f) := \ew[\DVertn]{\loss(X,Y,f(X))} = \frac{1}{n}  \sum_{i=1}^{n} \loss(x_i,y_i,f(x_i))\,.
\end{align*}

Because just minimizing \riskempn constitutes an ill-posed problem and usually results in some extent of overfitting, a regularization term has to be added. This leads to the definition of SVMs as minimizers of the regularized risk. More specifically, the \textit{empirical SVM} is defined as
\begin{align}\label{eq:Consistency_SVMs_empSVM}
	\fempn := \arg\inf_{f\in\H} \riskempn(f) + \lb ||f||_\H^2\,,
\end{align}
and the \textit{theoretical SVM} analogously as 
\begin{align}\label{eq:Consistency_SVMs_theoSVM}
	\ftheo := \arg\inf_{f\in\H} \risk(f) + \lb ||f||_\H^2\,.
\end{align}
In both definitions, $\lb>0$ is a {regularization parameter} which controls the amount of regularization and \H is the \textit{reproducing kernel Hilbert space} (RKHS) of a measurable \textit{kernel} on \X, \ie a symmetric and positive definite function $\k\colon \XX\to\R$, \vgl \citet{aronszajn1950,berlinet2004,saitoh2016} among others for a thorough introduction to this topic. We are often be interested in {bounded kernels} for which we define $||\k||_\infty := \sup_{x\in\X} \sqrt{k(x,x)}$. Additionally, we define the \textit{canonical feature map} $\Phi\colon\X\to\H$ by $\Phi(x):=\k(\cdot,x)$.

SVMs have been widely investigated and have been shown to possess many desirable properties including existence, uniqueness, risk consistency, statistical robustness, and the existence of representation theorems under rather mild assumptions. See for example \citet{vapnik1995,vapnik1998,schoelkopf2002,cucker2007,steinwart2008,vanmessem2020} for a detailed introduction. More recent results regarding statistical robustness and stability in general have for example been derived by
\citet{hable2011,sheng2020,eckstein2022arxiv,koehler2022}. Results on learning rates \citep[which have to make more restrictive assumptions regarding \P because of the no-free-lunch-theorem, \vgl][]{devroye1982} can for example be found in \citet{caponnetto2007,steinwart2009,eberts2013,hang2017,fischer2020}.

\subsection{Consistency of regularized kernel methods based on regular loss functions}\label{SubSec:Consistency_SVMs_Reg}

Whereas SVMs based on distance-based losses are known to be risk consistent under mild assumptions \citep[\vgl][Theorem~12]{christmann2007}, there are no general results on their \Lp-consistency so far, but instead only corollaries for special loss functions based on the results mentioned at the beginning of \Cref{SubSec:Consistency_Consistency_Reg}.

Since the conditions required by \citet[Theorem~12]{christmann2007} also imply the validity of \Cref{Thm:Consistency_Consistency_RiskLp}, \Lp-consistency of such SVMs would now directly follow under these conditions. However, by some more thorough investigations, we are even able to slightly relax the conditions on the sequence $(\lbn)_{n\in\N}$ of regularization parameters, namely only requiring it to satisfy $\lbn^{p^*}n\to \infty$ (as $n\to\infty$) for $p^*=\max\{p+1,p(p+1)/2\}$ instead of for $p^*=\max\{2p,p^2\}$, which is required by \citet[Theorem 12]{christmann2007}.

\begin{mythm}\label{Thm:Consistency_SVMs_LpCons}
	Let $\loss\colon\YR\to[0,\infty)$ be a convex, distance-based loss function of growth type $p\in[1,\infty)$. Let $\H\subseteq\Lppx$ dense and separable be the RKHS of a bounded and measurable kernel \k. Assume that \fbayes is \Px-\fs unique and $|\P|_p<\infty$. Define $p^*:=\max\{p+1,p(p+1)/2\}$. If the sequence $(\lb_n)_{n\in\N}$ satisfies $\lb_n>0$ for all $n\in\N$ as well as $\lb_n\to 0$ and $\lb_n^{p^*}n\to\infty$ for $n\to\infty$, then
	\begin{align*}
		\limn ||\fempnn-\fbayes||_{\Lppx} = 0 \qquad \text{in probability $\P^\infty$.}
	\end{align*}
\end{mythm}

\begin{mybem}
	The conditions on \H in \Cref{Thm:Consistency_SVMs_LpCons} can be difficult to check directly. However, if \X is separable, the separability of \H immediately follows whenever \k is continuous (\vgl \citeauthor{berlinet2004}, \citeyear{berlinet2004}, Corollary 4
	) and it suffices to verify this continuity instead. For example, the commonly used Gaussian RBF kernel (among many other kernels) satisfies this continuity, and since additionally its RKHS is dense in \Lppx \citep[\vgl][Theorem 4.63]{steinwart2008}, the RKHS satisfies both conditions from \Cref{Thm:Consistency_SVMs_LpCons}.
\end{mybem}

As we successfully slightly reduced the conditions regarding $(\lbn)_{n\in\N}$ compared to the referenced result on risk consistency, we can now transfer this slight relaxation back from \Lp-consistency to risk consistency by using \Cref{Thm:Consistency_Consistency_LpRisk}:

\begin{mykor}\label{Cor:Consistency_SVMs_RiskCons}
	Let $\loss\colon\YR\to[0,\infty)$ be a convex, distance-based loss function of growth type $p\in[1,\infty)$. Let $\H\subseteq\Lppx$ dense and separable be the RKHS of a bounded and measurable kernel \k. Assume that \fbayes is \Px-\fs unique and $|\P|_p<\infty$. Define $p^*:=\max\{p+1,p(p+1)/2\}$. If the sequence $(\lb_n)_{n\in\N}$ satisfies $\lb_n>0$ for all $n\in\N$ as well as $\lb_n\to 0$ and $\lb_n^{p^*}n\to\infty$ for $n\to\infty$, then
	\begin{equation*}
		\limn \risk(\fempnn) = \riskbayes \qquad \text{in probability $\P^\infty$.}
	\end{equation*}
\end{mykor}

Alas, the slight relaxation of the mentioned condition regarding the regularization parameters also comes along with an additional condition compared to \citet[Theorem 12]{christmann2007}: \Cref{Cor:Consistency_SVMs_RiskCons} requires \fbayes to be \Px-\fs unique. Thus, \Cref{Cor:Consistency_SVMs_RiskCons} pays for the slight relaxation in one condition by introducing this new additional condition and should therefore not be seen as a replacement of Theorem 12 from \citet{christmann2007} but as an addition instead.

\subsection{Consistency of regularized kernel methods based on shifted loss functions}\label{SubSec:Consistency_SVMs_Shift}

SVMs based on shifted loss functions can be defined analogously as in the non-shifted case in \eqref{eq:Consistency_SVMs_empSVM} and \eqref{eq:Consistency_SVMs_theoSVM}. \citet{christmann2009} proved that SVMs using Lipschitz continuous shifted loss functions inherit many of the desirable properties from their non-shifted counterparts, even without requiring the moment condition. These results include existence, uniqueness, representation and statistical robustness as well as risk consistency. Furthermore, they showed that $\fshifttheo=\ftheo$ whenever \ftheo uniquely exists. 

The natural hope that \Cref{Thm:Consistency_SVMs_LpCons} can be transferred to the shifted case similarly, thus also ridding it of the moment condition, might have already decreased because of the negative results from \Cref{SubSec:Consistency_Consistency_Shift}. As SVMs are always contained in some RKHS \H, one might however still hope that counterexamples like the ones from these results' proofs can not occur in such RKHSs because of the additional structure they possess compared to \Lapx.\footnote{The associated kernel \k being bounded and measurable implies that all $f\in \H$ are bounded and measurable as well, and hence that $\H\subseteq\Lapx$, \vgl \citet[Lemma~4.23~and~4.24]{steinwart2008}.}
Alas, Sobolev spaces like the ones considered in \Cref{Cor:Consistency_Consistency_ShiftGegenbeispielSobolevRiskLp} are also RKHSs if one chooses a suiting kernel like for example the ones found in \citet{wu1995,wendland2005}, which are classical examples of kernels with compact support. Hence, we obtain the following:

\begin{mykor}\label{Cor:Consistency_SVMs_ShiftGegenbeispielRKHS}
	Let $\loss\colon\YR\to[0,\infty)$ be a convex, distance-based and symmetric loss function of growth type 1, or the $\tau$-pinball loss for some $\tau\in(0,1)$. Let \lossshift be its shifted version. Then, even if \H is the RKHS of a bounded and measurable kernel \k, \fshiftbayes is \Px-\fs unique with $\fshiftbayes\in\Lapx$, a sequence $(f_n)_{n\in\N}\subseteq H$ of functions satisfying 
	\begin{align*}
		\limn \riskshift(f_n) = \riskshiftbayes
	\end{align*}
	does in general \textbf{not} imply 
	\begin{align*}
		\limn \normLapx{f_n - \fshiftbayes} = 0
	\end{align*}
	without any additional assumptions besides \Cref{Ann:Consistency_Pre_AllgAnn} being imposed.
\end{mykor}

As the (probably) most commonly used RKHSs for computing SVMs are those of the \textit{Gaussian RBF kernels} $\k_\gamma$, $\gamma\in(0,\infty)$, defined by
\begin{align*}
	\k_\gamma(x,x') := \exp\left(-\frac{\norm{2}{x-x'}^2}{\gamma^2}\right) \qquad \forall\, x,x'\in\X\,,
\end{align*}
we also want to take a special look at these. After proving in \Cref{Cor:Consistency_SVMs_ShiftGegenbeispielRKHS} that RKHSs, in which \La-consistency does not follow from risk consistency, do in fact exist, we see in the subsequent \Cref{Cor:Consistency_SVMs_ShiftGegenbeispielGauss} that this phenomenon can not only occur for kernels whose RKHS is a Sobolev space but also for that of the Gaussian RBF kernel.

\begin{mykor}\label{Cor:Consistency_SVMs_ShiftGegenbeispielGauss}
	Let $\loss\colon\YR\to[0,\infty)$ be a convex, distance-based and symmetric loss function of growth type 1, or the $\tau$-pinball loss for some $\tau\in(0,1)$. Let \lossshift be its shifted version. Let $\gamma\in(0,\infty)$ and $\H_\gamma$ be the RKHS of the Gaussian RBF kernel $\k_\gamma$. Then, even if \fshiftbayes is \Px-\fs unique with $\fshiftbayes\in\Lapx$, a sequence $(f_n)_{n\in\N}\subseteq \H_\gamma$ of functions satisfying 
	\begin{align*}
		\limn \riskshift(f_n) = \riskshiftbayes
	\end{align*}
	does in general \textbf{not} imply 
	\begin{align*}
		\limn \normLapx{f_n - \fshiftbayes} = 0
	\end{align*}
	without any additional assumptions besides \Cref{Ann:Consistency_Pre_AllgAnn} being imposed.
\end{mykor}

The previous results show that \La-consistency of SVMs using shifted loss functions does in general not follow from their risk consistency, with the latter being known from \citet[Theorem~8]{christmann2009}. Note that it might still be possible for such SVMs to be \La-consistent for different reasons though. 

At least in the special case of the shifted pinball loss, we found some alternative conditions to replace---and in many situations weaken---the moment condition from \Cref{Thm:Consistency_Consistency_ShiftPinRiskLa}. With this, we can now at least deduce \La-consistency of SVMs using this shifted pinball loss without needing to impose the moment condition:

\begin{mykor}\label{Cor:Consistency_SVMs_ShiftPinLaCons}
	Let $\tau\in(0,1)$ and \losspinshift be the shifted $\tau$-pinball loss. Let $\H\subseteq\Lapx$ dense and separable be the RKHS of a bounded and measurable kernel \k. Assume that \ftaubayes is \Px-\fs unique, $\ftaubayes\in \Lapx$ and \P additionally satisfies at least one of the additional conditions (i) and (ii) from \Cref{Thm:Consistency_Consistency_ShiftPinRiskLa}. If the sequence $(\lb_n)_{n\in\N}$ satisfies $\lb_n>0$ for all $n\in\N$ as well as $\lb_n\to 0$ and $\lb_n^2n\to\infty$ for $n\to\infty$, then
	\begin{equation*}
		\limn ||\ftaushiftempn-\ftaubayes||_{\Lapx} = 0 \qquad \text{in probability $\P^\infty$.}
	\end{equation*}
\end{mykor}

\begin{mybem}
	It would be possible to use \Cref{Cor:Consistency_SVMs_ShiftPinLaCons} to derive a result on risk consistency of SVMs which are based on the shifted pinball loss, similarly to what we did in the non-shifted case in \Cref{SubSec:Consistency_SVMs_Reg}, where we used \Cref{Thm:Consistency_SVMs_LpCons} to derive \Cref{Cor:Consistency_SVMs_RiskCons}. In the latter result, we however only achieved an actual improvement (over already existing results) regarding the conditions on the regularization parameters if the loss function is of growth type $p>1$. Similarly, a result on risk consistency which is based on \Cref{Cor:Consistency_SVMs_ShiftPinLaCons} would offer no benefit over Theorem 8 from \citet{christmann2009} because of the pinball loss being of growth type 1. 
\end{mybem}

\section{Discussion}\label{Sec:Consistency_Discussion}

This paper considerably generalized existing results regarding the close relationship between \Lp- and risk consistency by deriving results which are applicable to a wide range of loss functions. We additionally tried to eliminate the moment condition from the results connecting \Lp- and risk consistency by switching to shifted loss functions. Somewhat surprisingly, this only worked for one of the two directions (risk consistency following from \Lp-consistency), but in general not for the reverse. We proved that it is indeed not possible to infer \Lp-consistency from risk consistency if neither some standard moment condition nor some suitable alternative condition holds true. 

In case of using the shifted pinball loss, which can be used for quantile regression, we derived such an alternative condition, which is in many cases considerably weaker than the moment condition, thus still gaining some benefit from switching to shifted loss functions. It remains to be seen whether similar alternative conditions can also be derived for different loss functions or whether it might even be possible to derive a general alternative condition applicable to a wider array of loss functions. 

Lastly, we applied our results to regularized kernel methods. By doing so, we proved their \Lp-consistency in considerably greater generality than it had been known so far, and we slightly reduced a condition from results on their risk consistency from the literature.

\section*{Acknowledgments}
I would like to thank Andreas Christmann for helpful discussions on this topic.

\appendix

\section{Proofs}

\subsection{Proofs for Section \ref{SubSec:Consistency_Consistency_Reg}}

\begin{proof}[Proof of \Cref{Thm:Consistency_Consistency_RiskLp}]
	Let $g_n\colon \XY\to[0,\infty), (x,y)\mapsto\loss(y,\fn(x))$ for $n\in\N$, and $g^*\colon\XY\to[0,\infty), (x,y)\mapsto\loss(y,\fbayes(x))$. According to 
	\citet[Corollary~3.62]{steinwart2008}---where it is easy to see that we do not need the assumption of the sets $\mathcal{M}_{\loss,\Pbed{x},x}$ being singletons since we already know that \fbayes \Px-\fs uniquely exists\textemdash, we have $\fn\xlongrightarrow{\Px}\fbayes$. Thus, because of the continuous mapping theorem and the continuity of \loss, we also have $g_n\xlongrightarrow{\P}g^*$. Since
	\begin{align}\label{eq:ProofThm_Consistency_Consistency_RiskLp_LGleichgrInt}
		\limn \int |g_n|\diff\P = \limn \int g_n \diff\P &= \limn \risk(\fn)\notag\\ 
		&= \risk(\fbayes) = \int g^* \diff\P = \int |g^*| \diff\P\,,
	\end{align}
	the sequence $(|g_n|)_{n\in\N}$ is thus equi-integrable according to \citet[Theorem 21.7]{bauer2001}. That theorem can be applied because $\risk(\fbayes)<\infty$, and hence $\risk(\fn)<\infty$ for $n$ sufficiently large because of \eqref{eq:ProofThm_Consistency_Consistency_RiskLp_LGleichgrInt}, and therefore $g^*\in\Lapx$ and $g_n\in\Lapx$ for $n$ sufficiently large.
	
	Because of \loss being of lower growth type $p$, there now exists a constant $c>0$ such that
	\begin{align}\label{eq:ProofThm_Consistency_Consistency_RiskLp_fGleichgrInt}
		|\fn(x)-\fbayes(x)|^p &\le \max\left\{(2|y-\fn(x)|)^p \,,\, (2|y-\fbayes(x)|)^p \right\}\notag\\ 
		&\le 2^p\cdot \max\left\{c^{-1}\big(\loss(y,\fn(x))+1\big)\,,\, c^{-1}\big(\loss(y,\fbayes(x))+1\big) \right\}\notag\\
		&= \frac{2^p}{c} \cdot \big( \max\left\{g_n(x,y), g^*(x,y) \right\} +1\big)\notag\\
		&\le \frac{2^p}{c} \cdot \big( g_n(x,y) + g^*(x,y) +1\big) \qquad \forall\, (x,y,n)\in\XY\times\N\,,
	\end{align}
	since $g_n$, $n\in\N$, and $g^*$ are non-negative.
	
	As $(|g_n|)_{n\in\N}$ is equi-integrable, and $g^*\in\Lapx$ and hence also equi-integrable \citep[\vgl][part 2 of the example on \seite 122]{bauer2001}, every summand occurring on the right hand side of \eqref{eq:ProofThm_Consistency_Consistency_RiskLp_fGleichgrInt} is equi-integrable (as a sequence in $n$). By employing the example on \seite 121 of \citet{bauer2001} as well as Corollary 21.3 from the same book, we hence obtain equi-integrability of the whole right hand side (as a sequence in $n$).
	
	Thus, the sequence $(|\fn-\fbayes|^p)_{n\in\N}$ is equi-integrable as well and \Lp-convergence of \fn to \fbayes, follows from \citet[Theorem 21.7]{bauer2001}. 
\end{proof}

\begin{proof}[Proof of \Cref{Thm:Consistency_Consistency_LpRisk}]
	Since $\normLppx{f_n-f^*}\to0$, we also have $f_n\xlongrightarrow{\Px}f^*$, and \citet[Theorem 21.7]{bauer2001} yields equi-integrability of the sequence $(|f_n|^p)_{n\in\N}$. Let $g_n\colon\XY\to[0,\infty), (x,y)\mapsto\loss(y,f_n(x))$ for $n\in\N$, and $g^*\colon\XY\to[0,\infty), (x,y)\mapsto\loss(y,f^*(x))$. Because of \loss being of upper growth type $p$, there then exists a $c>0$ such that 
	\begin{align}\label{eq:ProofThm_Consistency_Consistency_LpRisk_gleichgrInt}
		|g_n(x,y)| = g_n(x,y) = \loss(y,f_n(x)) &\le c \cdot \left( \left|y-f_n(x) \right|^p + 1 \right)\notag\\ 
		&\le c \cdot \left( 2^p\cdot \left(|y|^p + |f_n(x)|^p \right) + 1 \right) 
	\end{align}
	for all $ (x,y,n)\in\XY\times\N$.
	
	Since every summand on the right hand side of \eqref{eq:ProofThm_Consistency_Consistency_LpRisk_gleichgrInt} is equi-integrable (because $|\P|_p<\infty$), the whole right hand side is equi-integrable as well (as a sequence in $n$) by the example on \seite 121 of \citet{bauer2001} and Corollary 21.3 from the same book. Hence, the sequence $(|g_n|)_{n\in\N}$ is equi-integrable as well. 
	
	Additionally, $g_n\xlongrightarrow{\P}g^*$ because of $f_n\xlongrightarrow{\Px}f^*$ and the continuous mapping theorem in combination with the continuity of \loss, and thus, \citet[Theorem 21.7]{bauer2001} yields 
	\begin{equation*}
		\limn \risk(f_n) = \limn\int g_n\diff\P = \limn\int |g_n|\diff\P = \int |g^*|\diff\P = \int g^*\diff\P = \risk(f^*)\,. 
	\end{equation*}
\end{proof}

\subsection{Proofs for Section \ref{SubSec:Consistency_Consistency_Shift}}

Before proving \Cref{Prop:Consistency_Consistency_ShiftGegenbeispielRiskLp}, we first need the following auxiliary lemma:

\begin{mylem}\label{Lem:Consistency_Consistency_AuxRiskFinite}
	Let $\loss\colon\XYR\to[0,\infty)$ be a convex and Lipschitz continuous loss function, and let \lossshift be its shifted version. If there exists a measurable function $f\colon\X\to\R$ satisfying $\riskshift(f)=-\infty$, there also exists a measurable function $g\colon\X\to\R$ satisfying $\Px(g\ne0)>0$ and $\riskshift(g)\in(-\infty,0]$.
\end{mylem}

\begin{proof}
	If we denote the inner risk by
	\begin{align*}
		\innerriskshiftbed\colon\, \R\to\R\cup\{-\infty,+\infty\}\,,\, t\mapsto\int_\Y \lossshift(x,y,t)\diff\Pbed[y]{x}\,,
	\end{align*}
	we have 
	\begin{align*}
		\riskshift(f) &= \int \lossshift(x,y,f(x)) \diff\P(x,y)
		= \int \innerriskshiftbed(f(x))\diff\Px(x)\\
		&= \int \innerriskshiftbed^+(f(x))\diff\Px(x) - \int \innerriskshiftbed^-(f(x))\diff\Px(x)
		= -\infty\,,
	\end{align*}
	with $\innerriskshiftbed^+:=\max\{\innerriskshiftbed\,,\,0\}$ and $\innerriskshiftbed^-:=\max\{-\innerriskshiftbed\,,\,0\}$ denoting the positive and the negative part of $\innerriskshiftbed$ respectively. From the definition of the integral, we hence obtain
	\begin{align}\label{eq:ProofLem_Consistency_Proofs_AuxRiskFinite_NegativePart}
		\int \innerriskshiftbed^-(f(x))\diff\Px(x) = \infty\,
	\end{align}
	and therefore the existence of $c\in(0,\infty)$ and $A\subseteq\X$ measurable such that $\Px(A)>0$ and $\innerriskshiftbed^-(f(x)) \ge c$ for all $x\in A$. 
	
	We further know that $|\loss|_1>0$ because it is clear from the definition of Lipschitz continuous loss functions (\vgl \Cref{Bem:Consistency_Consistency_Lipschitz}) that $|\loss|_1=0$ would imply $\loss(x,y,f(x))=\loss(x,y,0)$ for all $(x,y)\in\XY$ and hence $\riskshift(f)=0$, which contradicts our assumptions. Therefore, \eqref{eq:ProofLem_Consistency_Proofs_AuxRiskFinite_NegativePart} directly implies that $|f(x)|\ge \frac{c}{|\loss|_1}>0$ for all $x\in A$ because otherwise
	\begin{align*}
		\innerriskshiftbed^-(f(x)) 
		&= \left(\int \lossshift(x,y,f(x)) \diff\Px(x)\right)^-
		\le \int \big|\lossshift(x,y,f(x))\big| \diff\Px(x)\\
		&= \int \big|\loss(x,y,f(x)) - \loss(x,y,0)\big| \diff\Px(x)
		\le |\loss|_1 \cdot |f(x)|
		< c\,,
	\end{align*}
	which would form a contradiction to $x$ coming from $A$.
	
	Define 
	\begin{align*}
		g(x) := \begin{cases}
			0 &\wenn x\notin A\,,\\
			\frac{c}{|\loss|_1}\cdot \text{sign}(f(x)) &\wenn x\in A\,.
		\end{cases}
	\end{align*}
	Then, $\Px(g\ne0)>0$ and
	\begin{align}\label{eq:ProofLem_Consistency_Proofs_AuxRiskFinite_Risk}
		\riskshift(g) &= \int_A \innerriskshiftbed(g(x)) \diff\Px(x) + \underbrace{\int_{\X\setminus A} \innerriskshiftbed(g(x)) \diff\Px(x)}_{=0}\,.
	\end{align}
	All that remains to investigate is the first integral on the right hand side. For all $x\in A$, we know that
	\begin{align*}
		\left|\innerriskshiftbed(g(x))\right| \le \int \left|\loss(x,y,g(x)) - \loss(x,y,0)\right| \diff\Pbed[y]{x} \le |\loss|_1 \cdot |g(x)| = c
	\end{align*}
	and
	\begin{align*}
		\innerriskshiftbed(g(x)) \le \max\left\{\innerriskshiftbed(0), \innerriskshiftbed(f(x)) \right\} = \innerriskshiftbed(0) = 0
	\end{align*}
	because $g(x)$ lies between 0 and $f(x)$, $\innerriskshiftbed(f(x))<0$ by definition of $A$, and \innerriskshiftbed is convex (which follows from \loss being convex). 
	
	Plugging this into the right hand side of \eqref{eq:ProofLem_Consistency_Proofs_AuxRiskFinite_Risk} yields $\riskshift(g)\in[-c,0]$ and hence the assertion.
\end{proof}

\begin{proof}[Proof of \Cref{Prop:Consistency_Consistency_ShiftGegenbeispielRiskLp}]
	We prove the statement by providing a counterexample.
	
	Because of \loss being of lower growth type 1,
	\begin{align*}
		c_0:=\sup\{r\in[0,\infty)\,|\,\psi(r)=0\} 
	\end{align*}
	is finite, where $\psi$ denotes the representing function belonging to \loss, as introduced in \Cref{Def:Consistency_Pre_DistBasedLoss}. Because of \loss being convex, distance-based, and symmetric, we have 
	\begin{align}\label{eq:ProofProp_Consistency_Consistency_ShiftGegenbeispielRiskLp_c0}
		\loss(y,t)=\psi(y-t)=0 \qquad \Leftrightarrow \qquad y-t\in[-c_0,c_0]\,. 
	\end{align}
	Assume without loss of generality that $c_0\le \frac{1}{2}$ (else just scale the subsequent example accordingly).
	
	Choose $\X:=(0,1)$, $\Px:=\unif$ and
	\begin{align}\label{eq:ProofProp_Consistency_Consistency_ShiftGegenbeispielRiskLp_Pbed}
		\Pbed{X=x} := x\cdot\unif[-1,1] + \frac{1-x}{2} \cdot \big(\dirac[-a_x] + \dirac[a_x]\big)\qquad \forall\, x\in\X\,,
	\end{align}
	where \unif[a,b] denotes the uniform distribution on $(a,b)$, \dirac[z] denotes the Dirac distribution in $z\in\R$ and $a_x>1$ is a constant depending on $x$ (and on \loss) that we will specify later on.\footnote{For the sake of strictly adhering to the completeness assumption from \Cref{Ann:Consistency_Pre_AllgAnn}, we can also choose \X as $[0,1]$ or \R, and \Pbed{X=x} as an arbitrary probability measure for $x\notin(0,1)$ without changing anything else.} Further define
	\begin{align}\label{eq:ProofProp_Consistency_Consistency_ShiftGegenbeispielRiskLp_fn}
		\fn\colon \X\to\R\,,\qquad x\mapsto\begin{cases}
			n &\wenn x\in\left(0,\frac{1}{n}\right)\,,\\
			0 &\sonst\,,
		\end{cases}
	\end{align}
	for $n\in\N$. As \fn is bounded for all $n\in\N$, we obviously have $(\fn)_{n\in\N}\subseteq\Lapx$. We now show that this example also possesses the remaining properties mentioned in the proposition, which consists of three main steps:
	
	First, we show that \fshiftbayes is \Px-\fs unique, more specifically $\fshiftbayes\equiv 0$ \Px-\fs, and $\fshiftbayes\in\Lapx$:\\
	Choose $\fstar\equiv0$. We show that $\riskshift(\fstar)<\riskshift(f)$ for all measurable $f\colon\X\to\R$ satisfying $\Px(f\ne 0)>0$. As $\riskshift(\fstar)=0$, the case $\riskshift(f)=\infty$ is trivial. Furthermore, if there was an $f$ satisfying $\riskshift(f)=-\infty$ and thus contradicting our claim, there would by \Cref{Lem:Consistency_Consistency_AuxRiskFinite} (which is applicable by \Cref{Bem:Consistency_Consistency_Lipschitz}) also exist a measurable $g$ with $\Px(g\ne0)>0$ and $-\infty<\riskshift(g)\le0=\riskshift(\fstar)$, which would also contradict our claim. Hence, we can without loss of generality assume that $\riskshift(f)\in\R$.\\
	Since $\fstar\equiv0$, we have, for each $x\in\X$ and $y\ge 0$,  
	\begin{align}\label{eq:ProofProp_Consistency_Consistency_ShiftGegenbeispielRiskLp_IntegrNonNeg}
		\lossshift\left(-y,\fstar(x)\right) + \lossshift\left(y,\fstar(x)\right) &= 2\cdot \lossshift(y,0)\notag\\ 
		&= 2\cdot \lossshift\left(y,\,\frac{1}{2}\cdot\left(-f(x)\right)+\frac{1}{2}\cdot f(x)\right)\notag\\
		&\le \lossshift\left(y,-f(x)\right) + \lossshift\left(y,f(x)\right)\notag\\ 
		&= \lossshift\left(-y,f(x)\right) + \lossshift\left(y,f(x)\right)
	\end{align}
	because of \loss being distance-based, symmetric and convex.\\
	Furthermore, by the definition of $f$, there exists $\eps:=(\eps_1,\eps_2)$ with $\eps_1,\eps_2>0$ such that $\Px(\X_\eps)>0$, where $\X_\eps:=\{x\in\X\,:\,|f(x)|\ge\eps_1 \text{ and } x\ge\eps_2\}$. Now, specifically look at $x\in\X_\eps$ and $y\in[c_0,c_0+\min\{\frac{1}{2},\frac{|f(x)|}{4}\}]\subseteq [0,1]$. First, only consider such $x$ that satisfy $f(x)>0$. We then obtain that
	\begin{align}\label{eq:ProofProp_Consistency_Consistency_ShiftGegenbeispielRiskLp_Distances}
		\left|-y-f(x)\right| = y + f(x) \ge c_0 + f(x)\quad\text{ and }\quad \left|\pm y-\fstar(x)\right| = y \le c_0+\frac{f(x)}{4}\,,
	\end{align}
	and hence
	\begin{align*}
		\loss(-y,f(x))\ge 4 \cdot \loss(-y,\fstar(x)) = 2\cdot \Big( \loss\left(y,\fstar(x)\right) + \loss\left(-y,\fstar(x)\right) \Big)
	\end{align*}
	because of \eqref{eq:ProofProp_Consistency_Consistency_ShiftGegenbeispielRiskLp_c0} and the convexity, symmetry and distance-basedness of \loss. Thus,
	\begin{align*}
		&\Big( \lossshift\left(-y,f(x)\right) + \lossshift\left(y,f(x)\right) \Big) - \Big( \lossshift\left(-y,\fstar(x)\right) + \lossshift\left(y,\fstar(x)\right) \Big) \\
		&= \Big( \loss\left(-y,f(x)\right) + \loss\left(y,f(x)\right) \Big) - \Big( \loss\left(-y,\fstar(x)\right) + \loss\left(y,\fstar(x)\right) \Big)\\
		&\ge \frac{1}{2} \cdot \loss\left(-y,f(x)\right) = \frac{1}{2} \cdot \psi(|-y-f(x)|) \ge \frac{1}{2} \cdot \psi(c_0 + f(x))\,,
	\end{align*}
	where, in the last step, we again applied the convexity and symmetry of \loss, as well as \eqref{eq:ProofProp_Consistency_Consistency_ShiftGegenbeispielRiskLp_Distances}.\\
	By interchanging the roles of $y$ and $-y$ in the preceding paragraph, we obtain an analogous inequality for the case that $f(x)<0$. Combining these two cases yields that
	\begin{align}\label{eq:ProofProp_Consistency_Consistency_ShiftGegenbeispielRiskLp_IntegrPos}
		&\Big( \lossshift\left(-y,f(x)\right) + \lossshift\left(y,f(x)\right) \Big) - \Big( \lossshift\left(-y,\fstar(x)\right) + \lossshift\left(y,\fstar(x)\right) \Big)\notag\\ 
		&\ge \frac{1}{2} \cdot \psi(c_0 + |f(x)|)\,
	\end{align}
	for all $x\in\X_\eps$ and $y\in[c_0,c_0+\min\{\frac{1}{2},\frac{|f(x)|}{4}\}]\subseteq [0,1]$.\\
	Because $\riskshift(\fstar)=0\in\R$ by the definition of \fstar and $\riskshift(f)\in\R$ by assumption, our considerations yield
	\begin{align*}
		&\riskshift(f) - \riskshift(\fstar)\\ 
		&= \int_\X \int_\Y \lossshift\left(y,f(x)\right) - \lossshift\left(y,\fstar(x)\right) \diff\Pbed[y]{x}\diff\Px(x)\\
		&= \int_\X \int_{[0,\infty)} \Big( \lossshift\left(-y,f(x)\right) + \lossshift\left(y,f(x)\right) \Big)\\ 
		&\hspace*{4cm}- \Big( \lossshift\left(-y,\fstar(x)\right) + \lossshift\left(y,\fstar(x)\right) \Big) \diff\Pbed[y]{x}\diff\Px(x)\\
		&\overset{\eqref{eq:ProofProp_Consistency_Consistency_ShiftGegenbeispielRiskLp_IntegrNonNeg}, \eqref{eq:ProofProp_Consistency_Consistency_ShiftGegenbeispielRiskLp_IntegrPos}}{\ge} \int_{\X_\eps} \int_{[c_0,c_0+\min\{\frac{1}{2},\frac{|f(x)|}{4}\}]} \frac{1}{2} \cdot \psi(c_0+|f(x)|) \diff\Pbed[y]{x}\diff\Px(x)\\
		&= \int_{\X_\eps} \frac{x}{2} \cdot \min\left\{\frac{1}{2},\frac{|f(x)|}{4}\right\} \cdot \frac{1}{2} \cdot \psi(c_0+|f(x)|) \diff\Px(x)\\
		&\ge \Px(\X_\eps) \cdot \frac{\eps_2}{2} \cdot \min\left\{\frac{1}{2},\frac{\eps_1}{4}\right\} \cdot \frac{1}{2} \cdot \psi(c_0+\eps_1)\\
		&\overset{\eqref{eq:ProofProp_Consistency_Consistency_ShiftGegenbeispielRiskLp_c0}}{>} 0\,.
	\end{align*}
	In the second step, we multiplied the integrand by 2 for $y=0$, which does not change the value of the integral since $\Pbed[Y=0]{X=x}=0$ for all $x\in\X$. In the final steps, we additionally applied that \Pbed{X=x} has Lebesgue density $\frac{x}{2}$ on $[c_0,c_0+\min\{\frac{1}{2},\frac{|f(x)|}{4}\}]\subseteq [0,1]$, respectively the definition of $\X_\eps$.\\
	Hence, $\fshiftbayes\equiv 0$ \Px-\fs and thus also $\fshiftbayes\in\Lapx$.
	
	Next, we show that $\limn \riskshift(f_n) = \riskshiftbayes$:\\
	Recall the definition of $\fn$, $n\in\N$, from \eqref{eq:ProofProp_Consistency_Consistency_ShiftGegenbeispielRiskLp_fn}. For all $n\in\N$, we have $\fshiftbayes, \fn\in\Lapx$ and therefore $\riskshiftbayes=\riskshift(\fshiftbayes)\in\R$ and $\riskshift(\fn) \in\R$ by \eqref{eq:Consistency_Consistency_ShiftRiskFinite}. Hence, we can write
	\begin{align}\label{eq:ProofProp_Consistency_Consistency_ShiftGegenbeispielRiskLp_RiskDiff}
		&\riskshift(f_n) - \riskshiftbayes\notag\\ 
		&= \int_\X \int_\Y \lossshift\left(y,\fn(x)\right) - \lossshift\left(y,\fshiftbayes(x)\right) \diff\Pbed[y]{x}\diff\Px(x)\notag\\
		&= \int_\X \int_\Y \loss\left(y,\fn(x)\right) - \loss\left(y,\fshiftbayes(x)\right) \diff\Pbed[y]{x}\diff\Px(x)\notag\\
		&= \int_0^{1/n}  \int_{-1}^1 \frac{x}{2} \cdot \big( \loss\left(y,n\right) - \loss\left(y,0\right) \big) \diff y \diff x\notag\\ 
		&\hspace*{0.5cm}+ \int_0^{1/n} \frac{1-x}{2} \cdot \Big( \big( \loss\left(-a_x,n\right) + \loss\left(a_x,n\right) \big) - \big( \loss\left(-a_x,0\right) + \loss\left(a_x,0\right) \big) \Big) \diff x\,,
	\end{align}
	where we applied the definition of \fn, \fshiftbayes, and \P in the last step. We will now analyze the two integrals on the right hand side separately and show that they both converge to 0 as $n\to\infty$, starting with the first one:
	\begin{align*}
		&\left| \int_0^{1/n}  \int_{-1}^1 \frac{x}{2} \cdot \big( \loss\left(y,n\right) - \loss\left(y,0\right) \big) \diff y \diff x \right|\\ 
		&\le \int_0^{1/n}  \int_{-1}^1 \frac{x}{2} \cdot |\loss|_1 \cdot |n-0| \diff y \diff x = \frac{|\loss|_1}{2n} \xlongrightarrow{n\to\infty} 0 
	\end{align*}
	with \loss being Lipschitz continuous by \Cref{Bem:Consistency_Consistency_Lipschitz}.\\
	As for the second integral on the right hand side of \eqref{eq:ProofProp_Consistency_Consistency_ShiftGegenbeispielRiskLp_RiskDiff}:\\
	We take a look at the subdifferential $\partial\psi$ \citep[\vgl][Definition~1.9]{phelps1993} of the representing function $\psi$ of \loss. Because of the symmetry of \loss, we will without loss of generality only investigate $\partial\psi(r)$ for $r\in[0,\infty)$. Define 
	\begin{align*}
		z(r):=\sup\partial\psi(r) \in [0,\infty) \qquad \forall\, r\in[0,\infty)\,, 
	\end{align*}
	where $z(r)<\infty$ will follow from \eqref{eq:ProofProp_Consistency_Consistency_ShiftGegenbeispielRiskLp_cLTilde} and $z(r)\ge0$ follows from \loss being monotonically increasing on $[0,\infty)$ because of it being distance-based and convex. Furthermore, let $c_\loss$ be the constant from the definition of the upper growth type 1 of \loss, that is
	\begin{align*}
		\psi(r) \le c_\loss \cdot (|r|+1)  \qquad \forall\, r\in\R\,.
	\end{align*}
	Assume there was an $r_0\in[0,\infty)$ such that $z(r_0)>c_\loss$. Then, by the definition of the subdifferential, we would obtain
	\begin{align*}
		c_\loss \cdot (r+1) \ge \psi(r) \ge \psi(r_0) + z(r_0) \cdot (r-r_0) \qquad \forall\, r\in [0,\infty)
	\end{align*}
	and hence
	\begin{align*}
		r \le \frac{\psi(r_0)-z(r_0) r_0 - c_\loss}{c_\loss - z(r_0)} \qquad \forall\, r\in [0,\infty)\,,
	\end{align*}
	which is a contradiction because the right hand side is a constant in \R that is independent of $r$. Hence, $z$ is bounded by $c_\loss$. Because of \loss additionally being monotonically increasing on $[0,\infty)$, we obtain that
	\begin{align}\label{eq:ProofProp_Consistency_Consistency_ShiftGegenbeispielRiskLp_cLTilde}
		\tilde{c}_\loss := \lim_{r\to\infty} z(r) = \sup_{r\in[0,\infty)} z(r) \le c_\loss 
	\end{align}
	exists. \\
	We can therefore, for each $x\in(0,1)$, choose $r_x\in [0,\infty)$ such that
	\begin{align}\label{eq:ProofProp_Consistency_Consistency_ShiftGegenbeispielRiskLp_AbschDiff}
		0 \le \tilde{c}_\loss - z(r_x) \le x\,
	\end{align}
	and
	\begin{align}\label{eq:ProofProp_Consistency_Consistency_ShiftGegenbeispielRiskLp_ApproxLinear}
		\psi(r_x) + z(r_x) \cdot (r-r_x) \le \psi(r) \le \psi(r_x) + \tilde{c}_\loss \cdot (r-r_x) \qquad \forall\, r\in[r_x,\infty)\,.
	\end{align}
	Now choose $a_x$ in the definition of \Pbed{X=x} in \eqref{eq:ProofProp_Consistency_Consistency_ShiftGegenbeispielRiskLp_Pbed} as $a_x := r_x + \frac{1}{x}$ for all $x\in (0,1)$. Please note that $a_x>1$ for all $x\in (0,1)$. We obtain
	\begin{align*}
		&\loss(-a_x,n) + \loss(a_x,n) \\
		&= \psi\left(|-a_x-n|\right) + \psi\left(|a_x-n|\right)\\ 
		&= \psi\left(r_x+\frac{1}{x}+n\right) + \psi\left(r_x+\frac{1}{x}-n\right)\\
		&\in \left[ 2\cdot\psi(r_x) + z(r_x) \cdot \left( \frac{1}{x} + n + \frac{1}{x} - n \right) \,,\, 2\cdot\psi(r_x) + \tilde{c}_\loss \cdot \left( \frac{1}{x} + n + \frac{1}{x} - n \right) \right]\\
		&= \left[ 2 \cdot \left( \psi(r_x) + \frac{z(r_x)}{x} \right)  \,,\,  2 \cdot \left( \psi(r_x) + \frac{\tilde{c}_\loss}{x} \right)  \right] \qquad\qquad\forall\, n\in\N \,,\, x\in\Big(0,\frac{1}{n}\Big)\,,
	\end{align*}
	where we applied the symmetry of \loss as well as \eqref{eq:ProofProp_Consistency_Consistency_ShiftGegenbeispielRiskLp_ApproxLinear} combined with the fact that $\frac{1}{x}+n\ge0$ and $\frac{1}{x}-n\ge0$. Analogously, we obtain
	\begin{align*}
		&\loss(-a_x,0) + \loss(a_x,0) \\
		&= 2\cdot\psi\left(r_x+\frac{1}{x}\right)\\ 
		&\in \left[ 2 \cdot \left( \psi(r_x) + \frac{z(r_x)}{x} \right)  \,,\,  2 \cdot \left( \psi(r_x) + \frac{\tilde{c}_\loss}{x} \right)  \right] \qquad\forall\, x\in\Big(0,\frac{1}{n}\Big)\,.
	\end{align*}\\
	Plugging these results into the second integral on the right hand side of \eqref{eq:ProofProp_Consistency_Consistency_ShiftGegenbeispielRiskLp_RiskDiff} finally yields
	\begin{align*}
		&\left| \int_0^{1/n} \frac{1-x}{2} \cdot \Big( \big( \loss\left(-a_x,n\right) + \loss\left(a_x,n\right) \big) - \big( \loss\left(-a_x,0\right) + \loss\left(a_x,0\right) \big) \Big) \diff x \right| \\
		&\le \int_0^{1/n} \frac{1-x}{2} \cdot \left(  2 \cdot \left( \psi(r_x) + \frac{\tilde{c}_\loss}{x} \right) - 2 \cdot \left( \psi(r_x) + \frac{z(r_x)}{x} \right) \right) \diff x\\
		&= \int_0^{1/n} \frac{1-x}{2} \cdot \frac{2}{x} \cdot \left( \tilde{c}_\loss - z(r_x) \right) \diff x\\
		&\overset{\eqref{eq:ProofProp_Consistency_Consistency_ShiftGegenbeispielRiskLp_AbschDiff}}{\le} \int_0^{1/n} (1-x) \diff x = \frac{1}{n} - \frac{1}{2n^2} \xlongrightarrow{n\to\infty} 0\,,
	\end{align*}
	and thus $\limn \riskshift(f_n) = \riskshiftbayes$.
	
	Finally and as a last step, we have to show that $\limn \normLapx{\fn-\fshiftbayes} \ne 0$:
	\begin{align*}
		\limn \normLapx{\fn-\fshiftbayes} = \limn \int_0^{1/n} |n-0| \diff x = \limn 1 \ne 0\,.
	\end{align*}\qedhere
\end{proof}

\begin{proof}[Proof of \Cref{Prop:Consistency_Consistency_ShiftGegenbeispielPinRiskLp}]
	Similarly to \Cref{Prop:Consistency_Consistency_ShiftGegenbeispielRiskLp}, we prove the statement by providing a counterexample:
	
	Choose $\X:=(0,1)$, $\Y:=\R$, $\Px:=\unif$, and 
	\begin{align*}
		\Pbed{X=x} =&\  x \cdot \Big( \tau \cdot \unif[(-1,0)] + (1-\tau) \cdot \unif[(0,1)] \Big)\\ 
		&\ + (1-x) \cdot \Big( \tau \cdot \dirac[-1/x] + (1-\tau) \cdot \dirac[1/x] \Big) \qquad \forall\, x\in\X\,,
	\end{align*}
	where $\unif[(a,b)]$ denotes the uniform distribution on $(a,b)$ and $\dirac[z]$ denotes the Dirac distribution in $z\in\R$.\footnote{For the sake of strictly adhering to the completeness assumption from \Cref{Ann:Consistency_Pre_AllgAnn}, we can also choose \X as $[0,1]$ or $\R$, and $\Pbed{X=x}$ as an arbitrary probability measure for $x\notin(0,1)$ without changing anything else.} From this definition, we immediately obtain that $\ftaubayes\equiv0\in\Lapx$.
	
	Further define
	\begin{align*}
		\fn\colon \X\to\R\,,\qquad x\mapsto\begin{cases}
			n &\wenn x\in\left(0,\frac{1}{n}\right)\,,\\
			0 &\sonst\,,
		\end{cases}
	\end{align*}
	for all $n\in\N$. As \fn is bounded for all $n\in\N$, we obviously have $(\fn)_{n\in\N}\subseteq\Lapx$.
	
	Because of the occurring risks both being finite, \vgl \eqref{eq:Consistency_Consistency_ShiftRiskFinite}, and $\risk[\losspinshift,\P]^*=\risk[\losspinshift,\P](\ftaubayes)$, \vgl \eqref{eq:Consistency_Consistency_PinballBayesRisk}, we can for all $n\in\N$ write
	\begin{align}\label{eq:ProofProp_Consistency_Consistency_ShiftGegenbeispielPinRiskLp_Riskdiff}
		&\risk[\losspinshift,\P](f_n) - \risk[\losspinshift,\P]^*\notag\\
		&= \int_{(0,1)} \int_\R \losspinshift(y,\fn(x)) - \losspinshift(y,\ftaubayes(x)) \diff\Pbed[y]{x} \diff\Px(x)\,.
	\end{align}
	For \Px-almost all $x\in\X$, we can now further analyze the inner integral, applying that $\fn(x)\ge\ftaubayes(x)$, by 
	\begin{align}\label{eq:ProofProp_Consistency_Consistency_ShiftGegenbeispielPinRiskLp_InneRiskdiff}
		&\int_\R \losspinshift(y,\fn(x)) - \losspinshift(y,\ftaubayes(x)) \diff\Pbed[y]{x} \notag\\
		&=\int_\R \losspin(y,\fn(x)) - \losspin(y,\ftaubayes(x)) \diff\Pbed[y]{x} \notag\\
		&= \int_{\left(-\infty,\ftaubayes(x)\right)} (1-\tau) \cdot \left(\fn(x)-\ftaubayes(x)\right) \diff\Pbed[y]{x}\notag\\ 
		&\hspace*{0.5cm}+ \int_{\left[\ftaubayes(x),\fn(x)\right)} (-\tau) \cdot \left(\fn(x)-\ftaubayes(x)\right) + \left(\fn(x)-y\right) \diff\Pbed[y]{x}\notag\\
		&\hspace*{0.5cm}+ \int_{\left[\fn(x),\infty\right)} (-\tau) \cdot \left(\fn(x)-\ftaubayes(x)\right) \diff\Pbed[y]{x} \notag\\
		&= \int_{\left[\ftaubayes(x),\fn(x)\right)} \left(\fn(x)-y\right) \diff\Pbed[y]{x}\,.
	\end{align}
	In the last step, we employed that, for \Px-almost all $x\in\X$, we know from the definition of $\P$ that $\P(\{\ftaubayes(x)\}\,|\,x)=0$ and therefore $\P((-\infty,\ftaubayes(x))\,|\,x)=\tau$ and $\P([\ftaubayes(x),\infty)\,|\,x)=1-\tau$ by the definition of \ftaubayes.
	
	Plugging \eqref{eq:ProofProp_Consistency_Consistency_ShiftGegenbeispielPinRiskLp_InneRiskdiff} and the definition of \fn and \ftaubayes into \eqref{eq:ProofProp_Consistency_Consistency_ShiftGegenbeispielPinRiskLp_Riskdiff}, we obtain
	\begin{align*}
		\risk[\losspinshift,\P](f_n) - \risk[\losspinshift,\P]^*\
		&= \int_{\left(0,\frac{1}{n}\right)} \int_{[0,n)} \left(n-y\right) \diff\Pbed[y]{x} \diff\Px(x)\\
		&= \int_0^{\frac{1}{n}} \int_0^1 (n-y) \cdot x \cdot (1-\tau) \diff y \diff x\\
		&= (1-\tau) \cdot \frac{2n-1}{4n^2} \to 0\,,\qquad n\to\infty\,.
	\end{align*}
	
	On the other hand,
	\begin{align*}
		\normLapx{\fn-\ftaubayes} = \int_{0}^{\frac{1}{n}} |n-0| \diff x = 1 \not\to 0\,,\qquad n\to\infty\,,
	\end{align*}
	which completes the proof.
\end{proof}

\begin{proof}[Proof of \Cref{Cor:Consistency_Consistency_ShiftGegenbeispielSobolevRiskLp}]
	The assertion follows directly from the proof of \Cref{Prop:Consistency_Consistency_ShiftGegenbeispielRiskLp} respectively \Cref{Prop:Consistency_Consistency_ShiftGegenbeispielPinRiskLp} by changing the functions \fn, $n\in\N$, to 
	\begin{align*}
		\fn\colon \X\to\R\,,\qquad x\mapsto\begin{cases}
			n\cdot (1-nx)^m &\wenn x\in\left(0,\frac{1}{n}\right)\,,\\
			0 &\sonst\,.
		\end{cases}
	\end{align*}
	Since, for all $n\in\N$, \fn is bounded and $m$ times weakly differentiable, we obtain $(\fn)_{n\in\N}\subseteq\sobolevX[m,\infty]\cap\Lapx\subseteq\sobolevX\cap\Lapx$.\footnote{If \X is not chosen as $(0,1)$ but instead as $[0,1]$ or $\R$ in the proofs of \Cref{Prop:Consistency_Consistency_ShiftGegenbeispielRiskLp} and \Cref{Prop:Consistency_Consistency_ShiftGegenbeispielPinRiskLp}, it is obviously possible to extend the functions $\fn$, $n\in\N$, in such a way that they are still in $\sobolevX[m,\infty]\cap\Lapx$.}

	If we denote the functions from the mentioned proofs by $g_n$, $n\in\N$, we have $\fshiftbayes(x) \le \fn(x) \le g_n(x)$ for $\Px$-almost all $x\in\X$ because $\fshiftbayes=0$ \Px-\fs (with $\fshiftbayes=\ftaubayes$ \Px-\fs in the situation of $\lossshift=\losspinshift$ by the considerations prior to \Cref{Prop:Consistency_Consistency_ShiftGegenbeispielPinRiskLp}). 
	It is easy to see that the convexity of \loss and the definition of \fshiftbayes as a minimizer of \riskshift therefore implies $\riskshift(\fn)-\riskshiftbayes \le \riskshift(g_n)-\riskshiftbayes$, which then yields $\limn \riskshift(\fn)=\riskshiftbayes$.
	
	At the same time, we obtain
	\begin{align*}
		\normLapx{\fn-\fshiftbayes} = \int_{0}^{1/n} |n\cdot (1-nx)^m-0| \diff x = \frac{1}{m+1} \not\to 0\,,\qquad n\to\infty\,,
	\end{align*}
	which completes the proof.
\end{proof}

\begin{proof}[Proof of \Cref{Thm:Consistency_Consistency_ShiftPinRiskLa}]
	By \eqref{eq:Consistency_Consistency_ShiftRiskFinite}, both $\risk[\losspinshift,\P](\fn)$, $n\in\N$, and $\risk[\losspinshift,\P](\ftaubayes)$ are finite.
	
	If condition (i) is satisfied, we further obtain as in \Cref{Bem:Consistency_Consistency_RiskLpsAlternativbedingungen} that $\risk[\losspin,\P](0)$ and $\risk[\losspin,\P](\fn)$, for $n\in\N$, are finite, and therefore also $\riskbayes[\losspin,\P]$. As $\riskbayes[\losspin,\P]=\risk[\losspin,\P](\ftaubayes)$ and $\riskbayes[\losspinshift,\P]=\risk[\losspinshift,\P](\ftaubayes)$ by \eqref{eq:Consistency_Consistency_PinballBayesRisk} and \eqref{eq:Consistency_Consistency_ShiftPinballBayesRisk}, we hence obtain
	\begin{align*}
		\risk[\losspin,\P](\fn) = \risk[\losspinshift,\P](\fn) + \risk[\losspin,\P](0) \qquad \forall\,n\in\N
	\end{align*}
	and 
	\begin{align*}
		\riskbayes[\losspin,\P] &= \risk[\losspin,\P](\ftaubayes)\\ 
		&= \risk[\losspinshift,\P](\ftaubayes) + \risk[\losspin,\P](0) = \riskbayes[\losspinshift,\P] + \risk[\losspin,\P](0)\,.
	\end{align*}
	\Cref{Thm:Consistency_Consistency_RiskLp} and \Cref{Bem:Consistency_Consistency_RiskLpsAlternativbedingungen} then yield the assertion because of \losspin being of growth type 1. Thus, it is only left to show that condition (ii) yields the assertion as well:
	
	Because of the finiteness of $\risk[\losspinshift,\P](\fn)$, $n\in\N$, and $\risk[\losspinshift,\P](\ftaubayes)$, the assumed risk consistency implies that the \P-integral of $\losspinshift(y,\fn(x))-\losspinshift(y,\ftaubayes(x))$ converges to 0 as $n\to\infty$. We will now begin by fixing an $x\in\X$ and further analyzing the inner integral with respect to $\Pbed{x}$:
	
	First, we look at the case that $\fn(x)\ge \ftaubayes(x)$. In this case, repeating the considerations from \eqref{eq:ProofProp_Consistency_Consistency_ShiftGegenbeispielPinRiskLp_InneRiskdiff}, where we can apply \eqref{eq:Thm_Consistency_Consistency_ShiftPinRiskLa_WMasse0} in the last step, yields for \Px-almost all such $x$ that
	\begin{align*}
		&\int_\Y \losspinshift(y,\fn(x)) -\losspinshift(y,\ftaubayes(x))\diff\P(y|x)\notag\\
		&\overset{\eqref{eq:ProofProp_Consistency_Consistency_ShiftGegenbeispielPinRiskLp_InneRiskdiff}}{=} \int_{\left[\ftaubayes(x),f_n(x)\right)}(f_n(x)-y)\diff\P(y|x)\notag\\
		&\ge \int_{\left[\ftaubayes(x),\frac{f_n(x)+\ftaubayes(x)}{2}\right)}(f_n(x)-y)\diff\P(y|x)\notag\\
		&\ge \left(f_n(x)-\frac{f_n(x)+\ftaubayes(x)}{2}\right) \cdot \P\left(\left.\left(\ftaubayes(x),\frac{f_n(x)+\ftaubayes(x)}{2}\right)\right|x\right)\notag\\
		&= \frac{f_n(x)-\ftaubayes(x)}{2} \cdot \P\left(\left.\left(\ftaubayes(x),\frac{f_n(x)+\ftaubayes(x)}{2}\right)\right|x\right)\,.
	\end{align*}
	
	If on the other hand $f_n(x)<\ftaubayes(x)$, we analogously obtain for \Px-almost all such $x$:
	\begin{align*}
		&\int_\Y \losspinshift(y,\fn(x)) -\losspinshift(y,\ftaubayes(x)) \diff\P(y|x)\\ 
		&\ge \frac{\ftaubayes(x)-f_n(x)}{2} \cdot \P\left(\left.\left(\frac{f_n(x)+\ftaubayes(x)}{2},\ftaubayes(x)\right)\right|x\right)\,.
	\end{align*}
	
	In summary,
	\begin{align}\label{eq:ProofThm_Consistency_Consistency_ShiftPinRiskLa_AbschaetzungZsmf}
		&\int_\Y \losspinshift(y,\fn(x)) -\losspinshift(y,\ftaubayes(x))\diff\P(y|X) \ge \frac{|f_n(X)-\ftaubayes(X)|}{2} \cdot \P\left(\left.J_{X,n}\right|X\right)\,
	\end{align}
	\Px-\fs, where $J_{x,n}:=\left(\min\left\{\ftaubayes(x),\frac{f_n(x)+\ftaubayes(x)}{2}\right\},\max\left\{\ftaubayes(x),\frac{f_n(x)+\ftaubayes(x)}{2}\right\}\right)$ for all $x\in\X$.
	
	Additionally, \citet[Corollary 31]{christmann2009} yields $f_n\xlongrightarrow{\Px}\ftaubayes$, \ie
	\begin{equation}\label{eq:ProofThm_Consistency_Consistency_ShiftPinRiskLa_KonvInW}
		\limn \Px(|f_n(X)-\ftaubayes(X)|>\varepsilon) = 0 \qquad \forall \eps>0\,.
	\end{equation}
	Now, let $\eps>0$ be an arbitrary positive number (without loss of generality $\eps<2 c_1$). \X can be partitioned as $\X=\bigcupdot_{i=1}^{3}\X_{i,\eps}$, where
	\begin{align*}
		&\X_{1,\eps}:=\left\{x\in\X: |f_n(x)-\ftaubayes(x)|\le\eps\right\},\notag\\ &\X_{2,\eps}:=\left\{x\in\X: \eps<|f_n(x)-\ftaubayes(x)|\le 2\cdot c_1\right\},\notag\\
		&\X_{3,\eps}:= \X_3 := \left\{x\in\X: |f_n(x)-\ftaubayes(x)|> 2\cdot c_1\right\},
	\end{align*}
	such that
	\begin{align}\label{eq:ProofThm_Consistency_Consistency_ShiftPinRiskLa_L1Norm}
		||f_n-\ftaubayes||_{\Lapx} = \sum_{i=1}^{3} \int_{\X_{i,\eps}} |f_n(x)-\ftaubayes(x)|\diff\Px(x)\,.
	\end{align}
	The three summands can now be analyzed separately:
	\begin{align*}
		&\int_{\X_{1,\eps}} |f_n(x)-\ftaubayes(x)|\diff\Px(x) \le \eps,\\
		&\int_{\X_{2,\eps}} |f_n(x)-\ftaubayes(x)|\diff\Px(x) \le 2\cdot c_1\cdot \Px(\X_{2,\eps}) \xlongrightarrow{\eqref{eq:ProofThm_Consistency_Consistency_ShiftPinRiskLa_KonvInW}} 0\,,\qquad n\to\infty\,,
	\end{align*}
	and
	\begin{align*}
		&\int_{\X_{3,\eps}} |f_n(x)-\ftaubayes(x)|\diff\Px(x)\notag\\ 
		&= \int_{\X_{3}} \left(\frac{|f_n(x)-\ftaubayes(x)|}{2}\cdot \P(J_{x,n}|x)\right)\cdot \frac{2}{\P(J_{x,n}|x)} \diff\Px(x)\notag\\
		&\overset{\eqref{eq:Thm_Consistency_Consistency_ShiftPinRiskLa_WMasseHerum},\eqref{eq:ProofThm_Consistency_Consistency_ShiftPinRiskLa_AbschaetzungZsmf}}{\le} \frac{2}{c_2} \cdot \int_{\X_{3}} \int_\Y \losspinshift(y,\fn(x)) -\losspinshift(y,\ftaubayes(x))\diff\P(y|x)\diff\Px(x)\\ 
		&\to 0\,,\qquad n\to\infty\,,
	\end{align*}
	with the last convergence holding true because
	\begin{align*}
		\int_\X \int_\Y \losspinshift(y,\fn(x)) -\losspinshift(y,\ftaubayes(x))\diff\P(y|x)\diff\Px(x)\to 0\,,\qquad n\to\infty\,,
	\end{align*}
	by assumption and 
	\begin{align*}
		\int_\Y \losspinshift(y,\fn(x)) -\losspinshift(y,\ftaubayes(x))\diff\P(y|X)\ge 0
	\end{align*}
	\Px-\fs by \eqref{eq:ProofThm_Consistency_Consistency_ShiftPinRiskLa_AbschaetzungZsmf}.
	
	Plugging these results into \eqref{eq:ProofThm_Consistency_Consistency_ShiftPinRiskLa_L1Norm} yields the assertion.
\end{proof}

\begin{proof}[Proof of \Cref{Thm:Consistency_Consistency_Shift_L1Risk}]
	We know from \eqref{eq:Consistency_Consistency_ShiftRiskFinite} that all risks appearing in this result are finite. \loss additionally being Lipschitz continuous (\vgl \Cref{Bem:Consistency_Consistency_Lipschitz}) yields
	\begin{align*}
		\left| \riskshift(\fn) - \riskshift(\fstar) \right| &\le \int \left| \lossshift(y,\fn(x)) - \lossshift(y,\fstar(x)) \right| \diff\P(x,y)\\ 
		&= \int \left| \loss(y,\fn(x)) - \loss(y,\fstar(x)) \right| \diff\P(x,y) \\
		&\le |\loss|_1 \cdot \int |\fn(x)-\fstar(x)|\diff\P(x,y)\\
		&= |\loss|_1 \cdot \normLapx{\fn-\fstar} \to 0\,\qquad n\to\infty\,.
	\end{align*}
\end{proof}

\subsection{Proofs for Section \ref{SubSec:Consistency_SVMs_Reg}}

\begin{proof}[Proof of \Cref{Thm:Consistency_SVMs_LpCons}]
	We can split up the difference, which we have to investigate, as
	\begin{align}\label{eq:ProofThm_Consistency_SVMs_LpCons_Aufteilung}
		\normLppx{\fempnn-\fbayes} &\le \normLppx{\fempnn-\ftheon} + \normLppx{\ftheon-\fbayes}\notag\\
		&\le \normSup{k} \normH{\fempnn-\ftheon} + \normLppx{\ftheon-\fbayes}\,
	\end{align}
	by \citet[Lemma~4.23]{steinwart2008}. We will now examine the two summands on the right hand side separately, starting with the first one:
	
	First, note that applying \citet[Lemma~4.23, equation~(5.4) and Lemma~2.38(i)]{steinwart2008} yields
	\begin{align}\label{eq:ProofThm_Consistency_SVMs_LpCons_AbschSup}
		\normSup{\ftheon} \le \normSup{\k} \cdot \normH{\ftheon} \le \normSup{\k} \cdot \risk(0)^{1/2} \cdot \lbn^{-1/2} \le c_{p,\loss,\P,\k} \cdot \lbn^{-1/2} 
	\end{align}
	for all $n\in\N$, with $c_{p,\loss,\P,\k}\in(0,\infty)$ denoting a constant depending only on $p$, \loss, \P and \k, but not on $\lbn$.
	
	We know from \citet[Corollary 5.11]{steinwart2008} that there exist functions $h_n\colon\XY\to\R$, $n\in\N$, such that 
	\begin{align}\label{eq:ProofThm_Consistency_SVMs_LpCons_DiffHNorm}
		\normH{\fempnn-\ftheon} \le \frac{1}{\lbn} \cdot \normH{\ew[\DVertn]{h_n\Phi}-\ew[\P]{h_n\Phi}} \qquad \forall \,n\in\N\,,
	\end{align}
	and, for $s:=p/(p-1)$,
	\begin{align}\label{eq:ProofThm_Consistency_SVMs_LpCons_hAbsch}
		\norm{\L[\P]{s}}{h_n} &\le 8^p \cdot c_\loss \cdot \left(1+ |\P|_p^{p-1} + \normSup{\ftheon}^{p-1} \right)\notag\\ 
		&\le 8^p \cdot  c_\loss \cdot \left(1+ |\P|_p^{p-1} + c_{p,\loss,\P,\k}^{p-1} \cdot \lbn^{-(p-1)/2} \right)\notag\\
		&\le \tilde{c}_{p,\loss,\P,\k} \cdot \lbn^{-(p-1)/2} \hspace*{4cm} \forall \,n\in\N\,,
	\end{align}
	where we employed \eqref{eq:ProofThm_Consistency_SVMs_LpCons_AbschSup} in the second and the boundedness of $(\lbn)_{n\in\N}$ in the third step, and where $c_\loss\in (0,\infty)$ and $\tilde{c}_{p,\loss,\P,\k} \in (0,\infty)$ denote constants depending only on $\loss$ respectively $p$, \loss, \P and \k.
	
	Now, we can apply \citet[Lemma 9.2]{steinwart2008} with $q:=p/(p-1)$ if $p>1$ and $q:=2$ if $p=1$, which leads to $q^*:=\min\{1/2,1-1/q\}=\min\{1/2,1/p\}=(p+1)/(2p^*)$, to the functions $h_n\Phi$, $n\in\N$: First of all, with the help of \eqref{eq:ProofThm_Consistency_SVMs_LpCons_hAbsch} we obtain
	\begin{align*}
		\norm{q}{h_n\Phi} := \left(\ew[\P]{\normH{h_n\Phi}^q}\right)^{1/q} \le \normSup{\k} \cdot \norm{\L[\P]{q}}{h_n} \le \normSup{\k} \cdot \tilde{c}_{p,\loss,\P,\k} \cdot \lbn^{-(p-1)/2} < \infty 
	\end{align*}
	for all $n\in\N$. We employed that, for all $(x,y)\in\XY$,
	\begin{align*}
		\normH{h_n(x,y)\Phi(x)}^q &= |h_n(x,y)|^q\cdot \normH{\Phi(x)}^q\notag\\ 
		&= |h_n(x,y)|^q \cdot \k(x,x)^{q/2} \le |h_n(x,y)|^q \cdot \normSup{\k}^q
	\end{align*}
	by the reproducing property \citep[\vgl for example][Definition~2.9]{schoelkopf2002}. Hence, we obtain for all $\eps>0$, by combining this Lemma 9.2 with \eqref{eq:ProofThm_Consistency_SVMs_LpCons_DiffHNorm},
	\begin{align*}
		&\P^n\left(\Dn\in(\XY)^n : \normH{\fempnn-\ftheon} \ge \eps\right) \\
		&\le \P^n\left(\Dn\in(\XY)^n : \normH{\ew[\DVertn]{h_n\Phi}-\ew[\P]{h_n\Phi}} \ge \lbn \cdot\eps\right)\\
		&\le c_q \cdot \left(\frac{\norm{q}{h_n\Phi}}{\lbn\eps n^{q^*}}\right)^q \le \hat{c}_{p,\loss,\P,\k} \cdot \left(\frac{1}{\lbn^{(p+1)/2}\eps n^{q^*}}\right)^q \to 0\,,\qquad n\to\infty\,,
	\end{align*}
	with $c_q\in(0,\infty)$ and $\hat{c}_{p,\loss,\P,\k}\in(0,\infty)$ denoting constants depending only on $q$ (that is, only on $p$) respectively $p$, \loss, \P and \k, and with the convergence in the last step holding true because 
	\begin{align*}
		\lbn^{(p+1)/2}n^{q^*} = \left(\lbn^{(p+1)/(2q^*)}n\right)^{q^*} = \left(\lbn^{p^*}n\right)^{q^*} \to \infty\,,\qquad n\to\infty\,,
	\end{align*}
	by the assumptions on $(\lbn)_{n\in\N}$. Thus, the first summand on the right hand side of \eqref{eq:ProofThm_Consistency_SVMs_LpCons_Aufteilung} converges to 0 in probability as $n\to\infty$.
	
	Now, we can turn our attention to the second summand: First of all, \citet[Lemma~2.38(i)]{steinwart2008} yields that \loss is a \P-integrable Nemitski loss of order $p$. Hence, we know from \citet[Theorem 5.31]{steinwart2008} that 
	\begin{align*}
		\riskoptH := \inf_{f\in\H}\risk(f) = \riskbayes\,,
	\end{align*}
	and \citet[Lemma 5.15]{steinwart2008} (with $\riskoptH=\riskbayes<\infty$ by \Cref{Bem:Consistency_Consistency_RiskLpsAlternativbedingungen}) then yields 
	\begin{align*}
		\limn \lbn \normH{\ftheon}^2 + \risk(\ftheon) - \riskbayes = 0
	\end{align*}
	because $\lbn\to 0$ as $n\to\infty$. Since $\lbn\normH{\ftheon}^2$ is non-negative and $\risk(\ftheon)\ge\riskbayes$ by the definition of \riskbayes, we obtain
	\begin{align*}
		\limn\risk(\ftheon) = \riskbayes\,.
	\end{align*}
	Hence, \Cref{Thm:Consistency_Consistency_RiskLp}, whose conditions are satisfied because of the considerations from \Cref{Bem:Consistency_Consistency_RiskLpsAlternativbedingungen}, yields convergence to 0 (as $n\to\infty$) of the second summand on the right hand side of \eqref{eq:ProofThm_Consistency_SVMs_LpCons_Aufteilung}, which completes the proof.
\end{proof}

\begin{proof}[Proof of \Cref{Cor:Consistency_SVMs_RiskCons}]
	The assertion follows directly from \Cref{Thm:Consistency_SVMs_LpCons} and \Cref{Thm:Consistency_Consistency_LpRisk}.
\end{proof}

\subsection{Proofs for Section \ref{SubSec:Consistency_SVMs_Shift}}

\begin{proof}[Proof of \Cref{Cor:Consistency_SVMs_ShiftGegenbeispielRKHS}]
	There exist different kernels whose RKHS is \sobolevX[2,2]. Examples of such kernels can be found in \citet{wu1995}, \citet[Chapter~7]{berlinet2004}, \citet[Theorem~1.11]{saitoh2016} among others. For this proof, we will however use the kernel $k_{1,1}$ defined by $k_{1,1}(x,x'):=\phi_{1,1}(||x-x'||_2)$ with $\phi_{1,1}$ as in \citet[Definition 9.11]{wendland2005}, that is $\phi_{1,1}(r)\propto (1-r)_+^3(3r+1)$ \citep[\vgl][Table 9.1]{wendland2005}. By \citet[Theorem 10.35]{wendland2005}, the RKHS of $k_{1,1}$ is indeed \sobolevX[2,2]. Additionally, $k_{1,1}$ is bounded by $\phi_{1,1}(0)<\infty$ and because of its continuity also measurable. 
	Applying \Cref{Cor:Consistency_Consistency_ShiftGegenbeispielSobolevRiskLp} yields the assertion.
\end{proof}

\begin{proof}[Proof of \Cref{Cor:Consistency_SVMs_ShiftGegenbeispielGauss}]
	Denote, for some $m\in\N$, the functions from the proof of \Cref{Cor:Consistency_Consistency_ShiftGegenbeispielSobolevRiskLp} by $g_n$, $n\in\N$. Because of $k_\gamma$ being universal, \vgl \citet[Definition 4.52 and Corollary 4.58]{steinwart2008}, and the functions $g_n$ being continuous, there exists a sequence $(\fn)_{n\in\N}\subseteq H_\gamma$ such that
	\begin{align*}
		\normSup{\fn-g_n} \le \frac{1}{n}
	\end{align*}
	for all $n\in\N$. 
	
	Since both $\fn$ and $g_n$ are bounded, we obtain from \eqref{eq:Consistency_Consistency_ShiftRiskFinite} that, for all $n\in\N$, $\riskshift(\fn)\in\R$ and $\riskshift(g_n)\in\R$. Hence,
	\begin{align*}
		&\left| \riskshift(\fn) - \riskshift(g_n) \right|
		\le \int_{\XY} \left| \lossshift(y,\fn(x)) - \lossshift(y,g_n(x)) \right| \diff\P(x,y)\\
		&= \int_{\XY} \left| \loss(y,\fn(x)) - \loss(y,g_n(x)) \right| \diff\P(x,y) \le |\loss|_1 \cdot \int_{\XY} \left|\fn(x) - g_n(x) \right| \diff\P(x,y)\\ 
		&\le |\loss|_1 \cdot \frac{1}{n} \to 0\,,\qquad n\to\infty\,.
	\end{align*}
	with \loss being Lipschitz continuous by \Cref{Bem:Consistency_Consistency_Lipschitz}.
	The risk consistency of $(g_n)_{n\in\N}$ shown in the proof of \Cref{Cor:Consistency_SVMs_ShiftGegenbeispielRKHS} then yields risk consistency of $(\fn)_{n\in\N}$.
	
	On the other hand,
	\begin{align*}
		\limn \normLapx{\fn-g_n} = \limn \int_\X \left| \fn(x) - g_n(x) \right| \diff\Px(x) \le \limn \frac{1}{n} = 0
	\end{align*}
	combined with 
	\begin{align*}
		\limn \normLapx{g_n-\fshiftbayes}=\frac{1}{m+1}\,,
	\end{align*}
	which is known from the proof of \Cref{Cor:Consistency_SVMs_ShiftGegenbeispielRKHS}, yields
	\begin{align*}
		\limn \normLapx{\fn-\fshiftbayes} \ge \limn \left( \normLapx{g_n-\fshiftbayes} - \normLapx{\fn-g_n} \right) = \frac{1}{m+1}
	\end{align*}
	and thus $(\fn)_{n\in\N}$ not being \La-consistent.
\end{proof}

\begin{proof}[Proof of \Cref{Cor:Consistency_SVMs_ShiftPinLaCons}]
	\citet[Theorem 8]{christmann2009} yields
	\begin{equation*}
		\limn \risk[\lossshift,\P](\ftaushiftempn) = \risk[\lossshift,\P](\ftaubayes)
	\end{equation*}
	in probability $\P^\infty$. The assertion follows directly from  \Cref{Thm:Consistency_Consistency_ShiftPinRiskLa}.
\end{proof}

\input{ms.bbl}
\end{document}

%% file: Makros.tex
\usepackage{xifthen}
\usepackage{xspace}

\usepackage{natbib}

\usepackage{etoolbox}
\makeatletter
\newcommand\bibstyle@comma{\bibpunct(),a,,}
\newcommand\bibstyle@semicolon{\bibpunct();a,,}
\makeatother
\pretocmd\citet{\citestyle{comma}}\relax\relax
\pretocmd\Citet{\citestyle{comma}}\relax\relax
\pretocmd\citep{\citestyle{semicolon}}\relax\relax
\pretocmd\Citep{\citestyle{semicolon}}\relax\relax

\newtheorem{mythm}{Theorem}
\numberwithin{mythm}{section}
\newtheorem{mylem}[mythm]{Lemma}
\newtheorem{myprop}[mythm]{Proposition}
\newtheorem{mykor}[mythm]{Corollary}
\theoremstyle{definition}
\newtheorem{mydef}[mythm]{Definition}
\newtheorem{mybsp}[mythm]{Example}
\newtheorem{myann}[mythm]{Assumption}

\theoremstyle{remark}
\newtheorem{mybem}[mythm]{Remark}

\numberwithin{figure}{section}

\crefname{mythm}{theorem}{theorems}
\crefname{mylem}{lemma}{lemmas}
\crefname{myprop}{proposition}{propositions}

\newcommand{\iid}{\text{i.i.d.}}
\newcommand{\ie}{\enm{\text{i.e.\@}}}
\newcommand{\vgl}{\enm{\text{cf.\@}}}
\newcommand{\seite}{\enm{\text{p.\@}}}
\newcommand{\seiten}{\enm{\text{pp.\@}}}
\newcommand{\eg}{\enm{\text{e.g.\@}}}
\newcommand{\wenn}{\text{, if }}
\newcommand{\sonst}{\text{, else}}
\newcommand{\fs}{\enm{\text{a.s.\@}}}
\newcommand{\enm}[1]{\ensuremath{#1}\xspace}
\newcommand{\limn}{\enm{\lim_{n\to\infty}}}
\newcommand{\diff}{\enm{\,\mathrm{d}}}
\newcommand{\R}{\enm{\mathbb{R}}}
\newcommand{\N}{\enm{\mathbb{N}}}
\newcommand{\eps}{\enm{\varepsilon}}
\newcommand{\lb}{\enm{\lambda}}
\newcommand{\lbn}{\enm{\lb_n}}
\renewcommand{\P}{\enm{\textnormal{P}}}
\newcommand{\Px}{\enm{\P_X}}
\newcommand{\Pbed}[2][\cdot]{\enm{\P(#1\,|\,#2)}} 

\newcommand{\D}{\enm{D}} 
\newcommand{\DVert}{\enm{\textnormal{D}}}
\newcommand{\Dn}{\enm{\D_n}} 
\newcommand{\DVertn}{\enm{\DVert_n}}
\newcommand{\ew}[2][]{\enm{\mathbb{E}_{#1}\left[#2\right]}}
\newcommand{\X}{\enm{\mathcal{X}}}
\newcommand{\Y}{\enm{\mathcal{Y}}}
\newcommand{\XX}{\enm{\X\times\X}}
\newcommand{\XY}{\enm{\X\times\Y}}
\newcommand{\XYR}{\enm{\X\times\Y\times\R}}
\newcommand{\YR}{\enm{\Y\times\R}}
\newcommand{\MXY}{\enm{\mathcal{M}_1(\XY)}}
\newcommand{\BX}{\enm{\mathcal{B}_{\X}}}
\newcommand{\BY}{\enm{\mathcal{B}_{\Y}}}
\newcommand{\BXY}{\enm{\mathcal{B}_{\XY}}}
\renewcommand{\L}[2][]{\enm{L_{#2}\ifthenelse{\isempty{#1}}{}{(#1)}}}
\newcommand{\Lp}{\L{p}}
\newcommand{\La}{\L{1}}
\newcommand{\Lppx}{\enm{\Lp(\Px)}} 
\newcommand{\Lapx}{\enm{\La(\Px)}} 
\newcommand{\sobolevX}[1][m,q]{\enm{W^{#1}(\X)}}
\newcommand{\norm}[2]{\enm{\left|\left|#2\right|\right|_{#1}}}
\newcommand{\normSup}[1]{\norm{\infty}{#1}}
\newcommand{\normH}[1]{\norm{\H}{#1}}
\newcommand{\normLppx}[1]{\norm{\Lppx}{#1}} 
\newcommand{\normLapx}[1]{\norm{\Lapx}{#1}} 
\renewcommand{\H}{\enm{H}} 
\renewcommand{\k}{\enm{k}}
\newcommand{\loss}{\enm{L}}
\newcommand{\lossshift}{\enm{\loss^\star}}
\newcommand{\losspin}{\enm{\loss_{\tau\text{-pin}}}}
\newcommand{\losspinshift}{\enm{\losspin^\star}}
\newcommand{\risk}[1][\loss,\P]{\enm{\mathcal{R}_{#1}}} 
\newcommand{\riskempn}[1][\loss,\DVertn]{\enm{\mathcal{R}_{#1}}} 
\newcommand{\riskshift}[1][\lossshift,\P]{\enm{\mathcal{R}_{#1}}} 
\newcommand{\riskbayes}[1][\loss,\P]{\enm{\mathcal{R}_{#1}^*}}
\newcommand{\riskshiftbayes}[1][\lossshift,\P]{\enm{\mathcal{R}_{#1}^*}} 
\newcommand{\riskoptH}[1][\loss,\P,\H]{\enm{\mathcal{R}_{#1}^*}} 
\newcommand{\innerrisk}[1][\loss,\Q]{\enm{\mathcal{C}_{#1}}} 
\newcommand{\innerriskshiftbed}[1][x]{\innerrisk[\lossshift,\Pbed{#1}]} 
\newcommand{\fbayes}{\enm{f_{\loss,\P}^*}}
\newcommand{\fshiftbayes}{\enm{f_{\lossshift,\P}^*}}
\newcommand{\ftaubayes}{\enm{f_{\tau,\P}^*}} 
\newcommand{\ftaubayeswithshiftloss}{\enm{f_{\losspinshift,\P}^*}} 
\newcommand{\fn}{\enm{f_n}}
\newcommand{\fstar}{\enm{f^*}}
\newcommand{\ftilde}[1][]{\ifthenelse{\isempty{#1}}{\enm{\tilde{f}}}{\enm{\tilde{f}_{#1}}}}
\newcommand{\ftheo}{\enm{f_{\loss,\P,\lb}}} 
\newcommand{\ftheon}{\enm{f_{\loss,\P,\lbn}}} 
\newcommand{\fempn}{\enm{f_{\loss,\DVert_n,\lb}}} 
\newcommand{\fempnn}{\enm{f_{\loss,\DVert_n,\lbn}}} 
\newcommand{\fshifttheo}{\enm{f_{\lossshift,\P,\lb}}} 
\newcommand{\ftaushiftempn}{\enm{f_{\losspinshift,\DVertn,\lbn}}} 
\newcommand{\unif}[1][\enm{0,1}]{\enm{\mathcal{U}(#1)}}
\newcommand{\dirac}[1][x]{\enm{\delta_{#1}}}
\newcommand{\bigcupdot}{\charfusion[\mathop]{\bigcup}{\cdot}}
\makeatletter
\def\moverlay{\mathpalette\mov@rlay}
\def\mov@rlay#1#2{\leavevmode\vtop{%
		\baselineskip\z@skip \lineskiplimit-\maxdimen
		\ialign{\hfil$\m@th#1##$\hfil\cr#2\crcr}}}
\newcommand{\charfusion}[3][\mathord]{
	#1{\ifx#1\mathop\vphantom{#2}\fi
		\mathpalette\mov@rlay{#2\cr#3}
	}
	\ifx#1\mathop\expandafter\displaylimits\fi}
\makeatother